\documentclass[twoside, 11pt]{article}
\usepackage{ewald-jmlr}
\usepackage[abbrvbib, preprint]{jmlr2e}

\newcommand\pen[1]{(#1){}^{+}} 

\renewcommand\P[2]{\mathcal{I}(#1,#2)} 

\newcommand{\SO}{O} 

\usepackage{lastpage}
\jmlrheading{26}{2025}{1-\pageref{LastPage}}{9/24; Revised 5/25}{6/25}{24-1516}{Thomas
Chen and Patr\'{i}cia Mu\~{n}oz Ewald}

\ShortHeadings{Interpretable Minima of Deep Neural Networks}{Chen and Ewald}

\begin{document}

\title{Interpretable Global Minima of Deep ReLU Neural Networks on Sequentially Separable
Data}

\editor{Christoph Lampert}

\author{Thomas Chen \email tc@math.utexas.edu \\
\addr Department of Mathematics \\ University of Texas at Austin \\ Austin, TX
78712 USA
\AND
\name Patr\'{i}cia Mu\~{n}oz Ewald\thanks{Corresponding author.} \email ewald@utexas.edu \\ 
\addr Department of Mathematics \\ University of Texas at Austin \\
Austin, TX 78712 USA 
}

\definecolor{burntorange}{HTML}{BF5700}
\definecolor{UTblue}{HTML}{00A9B7}
\definecolor{bluebonnet}{HTML}{005F86}
\hypersetup{ hidelinks,
    pdfauthor={Patricia Munoz Ewald}, 
    pdftitle={Interpretable minima of deep neural networks},
}


\maketitle

\begin{abstract}%
    We explicitly construct zero loss neural network classifiers.  We write the weight
    matrices and bias vectors in terms of cumulative parameters, which determine
    truncation maps acting recursively on input space. 
    The configurations for the training data considered are $(i)$ sufficiently small, well
    separated clusters corresponding to each class, and $(ii)$ equivalence classes which
    are sequentially linearly separable. In the best case, for $Q$ classes of data in
    $\mathbb{R}^{M}$, global minimizers can be described with $Q(M+2)$ parameters.
\end{abstract}


\section{Introduction}

Deep learning algorithms are often thought to involve a trade-off between accuracy and
interpretability.
Highly overparameterized neural networks are capable of interpolating
generic data, and gradient
descent offers a standard way of training the parameters to minimize some loss function.
However, the resulting models function as black boxes, providing limited insight into
their internal
mechanisms. In contrast, underparameterized networks cannot achieve zero training loss
unless the data possesses sufficient structure \citep{chenewald23gd,chenmoore25}.

In this work, continuing the program started by the authors
\citep{chenewald23deep,chenewald23shallow},
we provide an interpretation for the action of the layers of a neural
network, and explicitly construct global minima for a classification task with well
distributed data, in a manner that is independent of the number of training samples.
This characterization is done in terms of a map which encodes
the action of a layer given by a weight matrix $W$, a bias vector $b$, and activation
function $\sigma$, 
\begin{align*}
     \tau_{W,b}(x) = \pen{W}\left(\sigma (Wx + b) - b\right),
\end{align*}
where $\pen{W}$ denotes the generalized inverse of $W$.
A neural network with $L$ layers can be seen as an application of $L-1$ such
transformations $\tau$ on input space, which curate the data, followed by an affine
function in the last layer mapping to the output space,
\begin{align*}
       x^{(L)} &= W^{(L)} (x)^{(\tau, L-1)} + b^{(L)},
\end{align*}
where $x^{(\tau , L-1)}$ represents the $L-1$ consecutive applications of $\tau$ maps
to an input $x$, and $W^{(L)}, b^{(L)}$ are appropriate cumulative parameters.
While this works for any given activation function, we specialize to the case of a
ReLU neural network, and in this context call $\tau$ the \emph{truncation map}. 
An incomplete, but highly intuitive visualization of its action comes from looking at
cones:
For any cone in input space, there exist $W$ and $b$ such that $\tau_{W,b}$ acts as the
identity on the forward cone, and projects the entirety of the backward cone to the base
point (see Lemma \ref{lemma}).

\begin{figure}
    \label{fig:simplex}
    \centering
    \def\svgwidth{0.44\columnwidth}
    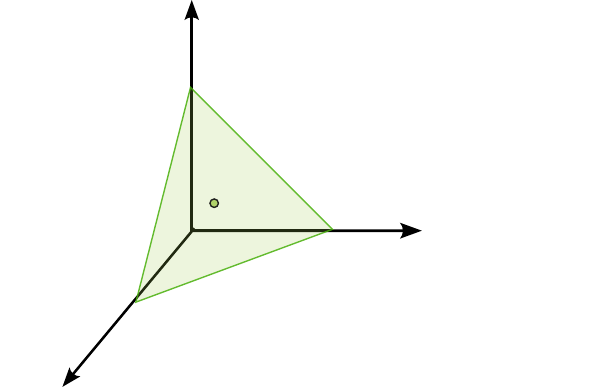
    \def\svgwidth{0.44\columnwidth}
    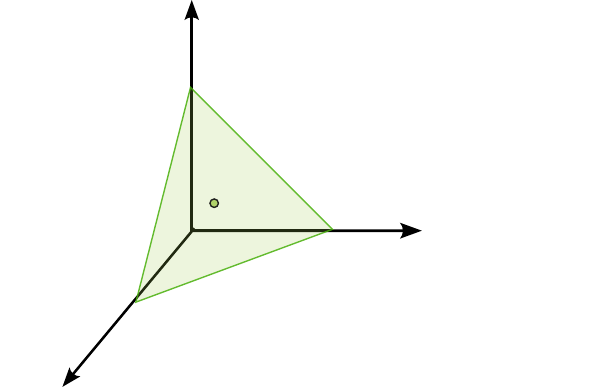
    \caption{A schematic representation of the action of a layer given by a truncation
    map. On the left, the original data with three classes contained in balls around the
    means $\overline{x_{0,j}}$, with $\mathcal{X}_{0,1}$
    sitting in the backward cone $\mathfrak{C}_{\theta}$, and the remaining classes
    sitting in the forward cone. On the right, the truncation map has pushed all of
    $\mathcal{X}_{0,1}$ to a single point, the base of the cone. In green, we see the
    2-simplex generated by the means, and the barycenter $\overline{x}$. The geometry of
the data corresponds to the assumptions of Proposition \ref{decreasingwidths}.}
\end{figure}

Consider data in $\mathbb{R}^{M}$ to be classified into $Q$ classes. 
In previous work \citep{chenewald23deep}, the tools outlined above were used to construct
explicit global minimizers for a network with constant width $M=Q$ and ``clustered'' data.
In this paper, we extend this
to the case where the width of the hidden layers is allowed to decrease monotonically from
$M$ to $Q$, for $Q\leq M$.
As a corollary, we show a reduction in the total number of parameters needed to
interpolate the data.
Our first result (informally stated) is

\begin{thm*}[Corollary to Proposition \ref{decreasingwidths}]
    Consider a set of training data $\mathcal{X}_{0} = \bigcup_{j=1}^{Q} \mathcal{X}_{0,j}
    \subset \mathbb{R}^{M}$ separated into $Q$ classes corresponding to linearly
    independent labels in $\mathbb{R}^{Q}$, for $Q \leq M$. If the data is distributed in
    sufficiently small and separated balls (``clustered''), then there exists a ReLU
    neural network with $Q(M+Q^{2})$ parameters in $Q+1$ layers which interpolates the
    data. 

    The weights and biases can be explicitly and geometrically described. Moreover, this
    does not depend on the size of the data set, $\abs{\mathcal{X}_{0}}$.
\end{thm*}

The proof can be easily sketched:
For each layer, a truncation map can be constructed which maps an entire cluster to a new
point, while keeping the remaining clusters the same (see Figure \ref{fig:simplex}). The
weights and biases can be given in terms of the variables defining the associated cones.

As the angular aperture of a cone gets larger, it becomes a hyperplane.
These are a well known classification tool, and there are algorithms dedicated to
finding hyperplanes maximizing the distance to the two sets being separated, such as
support vector machines (SVMs). Recent works show that under some circumstances, ReLU
networks trained with gradient descent to perform binary classification tend to find 
max-margin hyperplanes \citep[see e.g.,][]{phuonglampert20}.\footnote{The tendency of
    gradient-based algorithms to find critical points which generalize well or are simple
in some way, e.g. low-rank, is called implicit bias or implicit regularization.}

\begin{figure}
    \centering
    \def\svgwidth{0.85\columnwidth}
    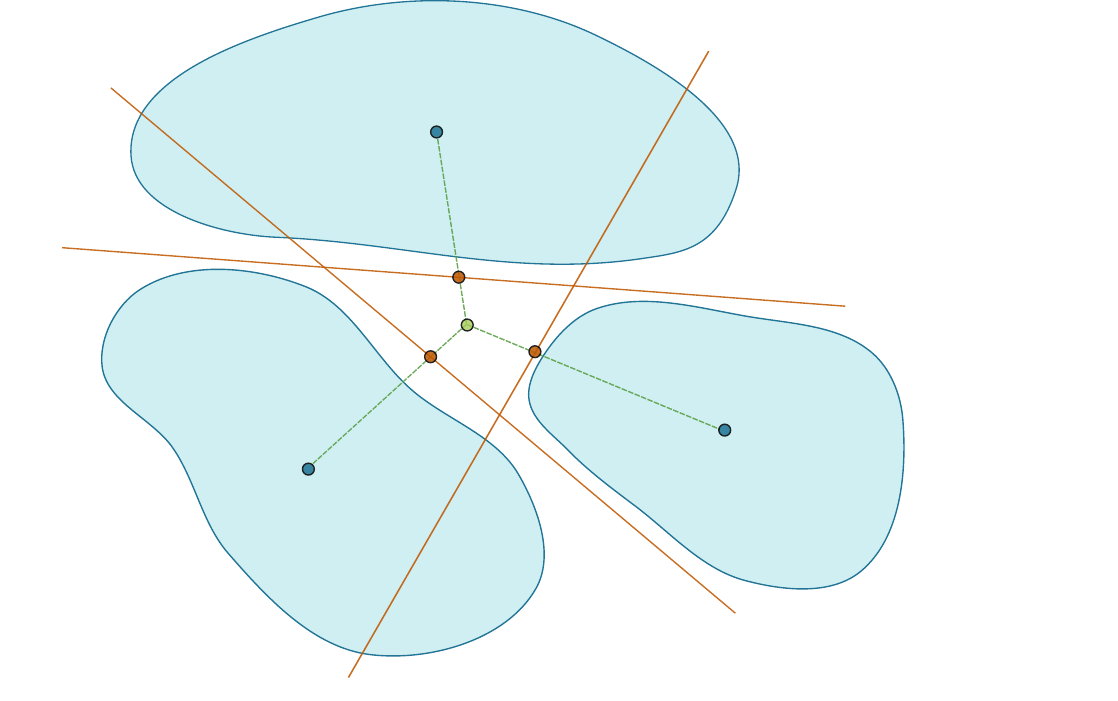
    \caption[a]{A 2-dimensional representation of a data set $\mathcal{X}_{0} =
        \bigcup_{j=1}^{3} \mathcal{X}_{0,j}$ that is sequentially linearly separable
        (Definition \ref{sls}), but
        not clustered
        (cf. Figure
        \ref{fig:simplex}). 
    Note that the ordering chosen is the only possible one. The geometry of
the data corresponds to the assumptions of Theorem \ref{mainthm}.}
    \label{fig:sls}
\end{figure}

The same algorithms can be adapted to the case of multiple classes, by reducing to binary
classification via several one-versus-all or one-versus-one strategies.  However, a
neural network tasked with performing multiclass classification must yield a single
function capable of classifying. When the data is clustered as above, Figure
\ref{fig:simplex} shows how the layers of the network will push each cluster into a
single point.  By considering hyperplanes as limits of cones, this
can be generalized to data for which at each step, one class gets mapped into a single
point. Thus, each layer acts to find a one-versus-all hyperplane. If this can be done so
that at the end there are $Q$ linearly independent points, the last layer will complete
the classification with zero loss. We define data which satisfies this condition as
\emph{sequentially linearly separable}, and this leads to our next result (informally
stated).

\begin{thm*}[Theorem \ref{mainthm}]
    Consider a set of training data $\mathcal{X}_{0} = \bigcup_{j=1}^{Q} \mathcal{X}_{0,j}
    \subset \mathbb{R}^{M}$ separated into $Q$ classes corresponding to linearly
    independent labels in $\mathbb{R}^{Q}$. If the data is sequentially
    linearly separable, then a ReLU neural network with $Q+1$ layers of size $d_0 = \cdots
    = d_{Q} = M \geq Q$, $d_{Q+1}=Q$, attains a degenerate global minimum with zero training
    cost, which can be parametrized by
    \begin{align*}
        \left\{(\theta_{\ell}, \nu_{\ell}, \mu_{\ell}) \right\}_{\ell =1}^{Q} \subset (0,
        \pi)  \times \mathbb{R}^{M} \times (0,1),
    \end{align*}
    corresponding to triples of angles, normal vectors and a line segment of base points
    (resp.) describing cones and hyperplanes. The parameters do not depend on the size of
    the data set, $\abs{\mathcal{X}_{0}}$.
\end{thm*}

In the work at hand, we focus on geometric structure in zero-loss neural networks,
and leave questions about algorithmic implementations to future work. We note that the
complexity of
algorithmic implementations will address the detection of the training data geometry,
instead of the determination of network parameters. Once the training data geometry is
known, the network parameters can be calculated non-algorithmically.

In the case $M=Q$, the inductive proof of the main result was adapted
to give $2^{Q}-1$ suboptimal local minima, associated to the binary choice for each
cluster to have the truncation map act trivially or collapse the data to
a point \citep{chenewald23deep}.  However, the proof of minimality relied on an exact
expression for the norm of local minima which holds in this case \citep[Theorem
3.2]{chenewald23shallow}. The
assumption $M=Q$ is dropped for our results here, so the local minima cannot be found in
the same way. Moreover, the sequential separability of the data in Theorem \ref{mainthm}
might depend on the ordering of the classes (cf. Figure \ref{fig:sls}), so that the
truncation maps cannot be chosen independently.

Finally, we note that while we do not consider layers with arbitrary width, increasing the
width would only add unnecessary parameters, as we assume the data to be suitably linearly
separable.  It is known that data which is not linearly separable can be classified by an
SVM with the use of kernel machines \citep{boseretal92} which map the data into a space
of higher dimension in such a way that, in this new space, a separating hyperplane can be
found. Moreover, neural networks approximate kernel machines in the infinite width limit
\citep{bahrietal18, jacotetal18}. With these considerations, it could be the case that
for data that is not linearly separable, the first layers of a neural network could be
seen as describing a map to a higher dimensional space where the conditions for sequential
linear separability are now satisfied, and the results above can be applied.
An implementation of this idea for binary classification is given by \citet{ewald25}.

\vspace{1em}

\subsection{Related Work} 
We give here a short and incomplete survey of related results.

Since the ReLU function acts as the identity when activated, it is instructive to compare
with the linear case.
The \textbf{loss landscape of linear neural networks} is very well understood. Its study
started with the classic work by \citet{baldihornik89}, which proved (under some assumptions
on the
architecture and data) that a shallow (one hidden layer) linear neural network with square
loss has only global minima and saddle points, and that every critical
point is a projection of the least-squares estimator onto a subspace spanned by the
eigenvectors of a relevant matrix determined by the data and label matrices.\footnote{That
    is, a critical point with $rank(W_2 W_1)=r$ satisfies $W_2 W_1 =
    U_{r}U_{r}^{T} YX^{+}$, where $U_{r}$ is formed by $r$ columns of $U$, the
    matrix of eigenvectors of $\Sigma = Y\pen{X}XY^{T}$. See next section for notation.} 
\citet{kawaguchi16} extended this to
deep linear networks, proving under milder conditions (both on data and architecture) that
there are no suboptimal local minima, and giving some properties of the Hessian of the
loss at saddle points. 
\citet{zhouliang18} provided complete characterizations of all critical points of a deep
linear network with square loss, arbitrary architecture and very mild assumptions on the
data. Recently, \citet{achouretal22} gave a complete second order analysis of the loss
landscape of deep linear networks, with the same assumptions as \citet{kawaguchi16}.

The \textbf{loss landscape of ReLU neural networks}, on the other hand, is not as well
understood.  Following assumptions adopted in work by \citet{choromanskaetal15}, 
\cite{kawaguchi16} treats the
activation functions as random variables, and what is computed is the loss
with expectation; under these assumptions, the loss reduces to the loss of a deep linear
network.  \citet{tian17} studied ReLU networks with no hidden layers, Gaussian distributed
input data and target function a neural network with fixed weights, and found that certain
types of critical points are not isolated.  \citet{zhouliang18} studied shallow ReLU
networks with very mild assumptions on the data, characterized all critical points of the
square loss for parameters in certain regions of parameter space (depending on the
activation pattern of the neurons, and not covering the whole space), and gave an example
of a network for which there is a suboptimal local minimum. While not studying critical
points of a loss function directly, \citet{grigsbyetal22, grigsbyetal23} studied functional
properties of ReLU neural networks, such as ``functional dimension''\footnote{A simplified
    definition of functional dimension is 
the rank of the Jacobian of the network with respect to its parameters.}, the level sets
of which are contained in level sets of any loss function. 

Finally, the results here show how a trained neural network could take 
within-class variability to zero.
This is one of the features of \textbf{neural collapse}, a phenomenon which has been
empirically shown by \citet{papyanhandonoho20, hanpapyandonoho22} to emerge when training a
neural network towards zero loss.

\subsection{Outline of Paper}
In Section \ref{sec:setting}, we clearly define the types of neural networks and the cost
function under consideration and fix notation. In Section \ref{sec:truncones}, we define
the truncation map, remark on its action in terms of cones in input space and cumulative
parameters, and prove key lemmas. In
Section \ref{sec:clusters}, we give precise statements for Proposition
\ref{decreasingwidths} and a corollary. In Section \ref{sec:hyp}, we define sequentially
linearly separable data, and present a precise statement and proof for Theorem \ref{mainthm}.
In Section \ref{sec:gen}, we briefly discuss generalization guarantees, and in Section
\ref{sec:nc}, we elaborate on the relationship between our results and neural collapse.
In Section \ref{sec:proof}, we prove Proposition \ref{decreasingwidths}.
Finally, we collect some basic results on geometry of cones in Appendix
\ref{appendix}.

\section{Setting and Notation} \label{sec:setting}

We will be dealing with neural networks of the form
\begin{align} \label{neuralnet}
    x^{(0)} &= x_{0} \in \mathbb{R}^{d_0} \,\, \text{ initial input,} \nonumber \\ 
    x^{(\ell )} &= \sigma(W_{\ell } x^{(\ell -1)} + b_{\ell }) \in \mathbb{R}^{d_{\ell}},
    \text{ for } \ell =1, \cdots, L-1, \\ 
    x^{(L)} &= W_{L}x^{(L-1)}+b_{L} \in \mathbb{R}^{d_{L}}, \nonumber
\end{align}
where $W_{\ell} \in \mathbb{R}^{d_{\ell }\times d_{\ell -1}}$ for $\ell =1, \cdots, L$
are weight matrices, $b_{\ell } \in \mathbb{R}^{d_{\ell }}$ for $\ell=1, \cdots, L$ are
bias vectors, and the activation function 
$\sigma$ is the ramp function which acts component-wise on vectors,
$\sigma(x)_{i}=(x_{i})_{+} = \max(x_{i},0)$.
When considering the action on a data matrix
$X_0$ with $N$ columns corresponding to $N$ data points, we write
\begin{align*}
    X^{(\ell)} = \sigma\left(W_{\ell} X^{(\ell-1)} + B_{\ell }\right),
\end{align*}
where
\begin{align*}
    B_{\ell } := b_{\ell} u_{N}^{T} \in \mathbb{R}^{d_{\ell}\times N}
\end{align*}
for 
\begin{align*}    
    u_{N} := (1, \cdots, 1)^{T} \in \mathbb{R}^{N}.
\end{align*}

Consider a set of training data $\mathcal{X}_0 = \bigcup_{i=1}^{Q} \mathcal{X}_{0,j}
\subset \mathbb{R}^{M}$, where points in each subset
$\mathcal{X}_{0,j}=\{x_{0,j,i}\}_{i=1}^{N_{j}}$ are associated to one of $Q$ linearly
independent labels, $y_{j} \in \mathbb{R}^{Q}$. Let $\sum_{j=1}^{Q} N_{j} =
N$, so that $\abs{\mathcal{X}_{0}}=N$. We will be interested in the average of training
inputs associated to each output $y_{j}$,
\begin{align*}
    \overline{x_{0,j}} := \frac{1}{N_{j}} \sum_{i=1}^{N_{j}} x_{0,j,i},
\end{align*}
and the differences
\begin{align*}
    \Delta x_{0,j,i} := x_{0,j,i} - \overline{x_{0,j}}.
\end{align*}

It will be convenient to arrange the data and related  information into matrices, and so we
make the following definitions,
\begin{alignat}{2} \label{X_0red}
    X_{0,j} &:= \left[x_{0,j,1} \cdots x_{0,j,i} \cdots x_{0,j,N_{j}}\right] &&\in
    \mathbb{R}^{M\times N_{j}}, \nonumber \\
    X_0 &:= \left[X_{0,1} \cdots X_{0,j} \cdots X_{0,Q}\right] &&\in \mathbb{R}^{M\times N},
    \nonumber \\
    \Delta X_{0,j} &:= \left[ \Delta x_{0,j,1} \cdots \Delta x_{0,j,i} \cdots \Delta
    x_{0,j,N_{j}}\right] &&\in \mathbb{R}^{M\times N_{j}}, \nonumber \\
        \Delta X_{0} &:= \left[ \Delta X_{0,1} \cdots \Delta X_{0,j} \cdots \Delta X_{0,Q}
        \right] &&\in \mathbb{R}^{M\times N}, \nonumber \\
        \overline{X_{0,j}} &:= X_{0,j} - \Delta X_{0,j} &&\in \mathbb{R}^{M\times N_{j}},
        \\
        \overline{X_0} &:= \left[ \overline{X_{0,1}} \cdots \overline{X_{0,j}} \cdots
        \overline{X_{0,Q}} \right] &&\in \mathbb{R}^{M\times N}, \nonumber \\
        \overline{X_{0}^{red}} &:= \left[\overline{x_{0,1}} \cdots \overline{x_{0,j}}
        \cdots \overline{x_{0,Q}} \right] &&\in \mathbb{R}^{M\times Q}, 
        \nonumber  \\ 
        Y &:= \left[y_1 \cdots y_{j} \cdots y_{Q}\right] &&\in
        \mathbb{R}^{Q\times Q}, \nonumber \\
        Y^{ext} &:= [ y_1 u_{N_1}^{T} \cdots y_{j} u_{N_{j}}^{T} \cdots
        y_{Q} u_{N_{Q}}^{T} ] &&\in \mathbb{R}^{Q\times N}. \nonumber 
\end{alignat}
Then $\mathcal{X}_{0,j} = \Col(X_{0,j})$, the set of columns of $X_{0,j}$, and similarly
$\mathcal{X}_{0} = \Col(X_0)$.

Throughout, we will assume that $\overline{X_0^{red}}$ and $Y$ have full rank $Q\leq M$,
i.e., that the averages $\overline{x_{0,j}}$ and output vectors $y_{j}$ are linearly
independent.  Since $\overline{X_0^{red}}$ is injective, we write its generalized inverse as
\begin{align*}
    \pen{\overline{X_0^{red}}} = (\overline{X_0^{red}}^{T}\overline{X_0^{red}})^{-1}
    \overline{X_0^{red}}^{T} \in \mathbb{R}^{Q \times M},
\end{align*}
and this is a left inverse of $\overline{X_0^{red}}$. 
In general, for any matrix $A \in \mathbb{R}^{m \times n}$, we let
$ \pen{A} \in \mathbb{R}^{n \times m}$
denote its generalized inverse, that is, the unique matrix satisfying
\begin{align}
    \label{inverse}
    A\pen{A}A &= A, \nonumber \\ 
    \pen{A}A\pen{A} &= \pen{A}, \\
    A\pen{A} \text{ and } \pen{A}A &\text{ are symmetric.} \nonumber
\end{align}
Recall that $\pen{A}A$ is an orthogonal projector to $(\ker(A))^{\perp}$, $A\pen{A}$ is an
orthogonal projector to $\ran(A)$, and $\pen{A}$ is a left (resp. right) inverse of $A$ if
it is injective (resp. surjective).

Finally, for an appropriate neural network defined as in \eqref{neuralnet} with $d_0=M$ and
$d_L=Q$, we consider the square loss, given by the $\mathcal{L}^{2}$ Schatten
(or Hilbert-Schmidt) norm, 
\begin{align*}
    \norm{A}_{\mathcal{L}^{2}} := \sqrt{\Tr(A A^{T})},
\end{align*}
so that the cost function is
\begin{align*}
    \mathcal{C}[X_0, Y^{ext}, (W_{i}, b_{i})_{i=1}^{L}] &:= \frac{1}{2} \norm{X^{(L)} -
    Y^{ext}}^{2}_{\mathcal{L}^{2}} \nonumber \\
                                                        &= \frac{1}{2} \sum_{j=1}^{Q}
                                                        \sum_{i=1}^{N_{j}}
                                                        \abs{x_{0,j,i}^{(L)} -
                                                    y_{j}}^{2},
\end{align*}
where $\abs{\, \cdot \,}$ denotes the usual euclidean norm. 
When the data set given by $X_0$ and $Y$ is clear by context, we will simply denote
\begin{align*} 
    \mathcal{C}[(W_{i}, b_{i})_{i=1}^{L}] := \mathcal{C}[X_0,
    Y^{ext}, (W_{i}, b_{i})_{i=1}^{L}].
\end{align*}

\section{Truncation Map, Cones and Cumulative Parameters} \label{sec:truncones}

We recall here the definition of the truncation map \citep{chenewald23deep}, making
the necessary modifications for the case where the width is not constant and the
weight matrices might not be square. 

\begin{defi} \label{def-tau}
   Given $W\in \mathbb{R}^{M_1 \times M_0}$ and $ b\in \mathbb{R}^{M_1}$ we define the
   truncation map
   \begin{align*}
        \tau_{W,b}:\mathbb{R}^{M_0} &\to \mathbb{R}^{M_0} \nonumber \\
        x &\mapsto \pen{W}\left(\sigma (Wx + b) - b\right),
   \end{align*}
   where $\pen{W}$ is the generalized inverse defined as in \eqref{inverse}.
\end{defi}
Similarly, we can define the action on a data matrix $X\in \mathbb{R}^{M_0\times N}$
\begin{align*}
    \tau_{W,b}(X) = \pen{W}\left(\sigma (WX + B) - B\right),
\end{align*}
where $B=b\,u_{N}^{T}$. 
This map encodes the action of a layer of a neural network and represents it in the
domain. 
It has a nice recursive property which makes it possible to extend this representation to
several layers.  The following is a small generalization of a previous result 
\citep{chenewald23deep} to the case where the dimension of the intermediate spaces given by
the hidden layers is allowed to decrease.

\begin{prop} \label{prop-tau}
   Assume $M = d_{0}\geq d_{1} \geq \cdots \geq d_{L} = Q$,
   $X^{(\ell)} \in \mathbb{R}^{d_{\ell }\times N}$
   corresponds to the output of a hidden layer of a neural network defined as in
   \eqref{neuralnet} on a data matrix $X_0 \in \mathbb{R}^{M\times N}$, and all of the
   associated weight matrices $W_{\ell} \in \mathbb{R}^{d_{\ell }\times d_{\ell
   -1}}$ are full rank.
   Then the truncation map defined satisfies
   \begin{align} \label{tau_1}
       X^{(\ell)} = W_{\ell}\, \tau_{W_{\ell},b_{\ell}}(X^{(\ell-1)}) + B_{\ell}.
   \end{align}
   Moreover, defining the cumulative parameters
   \begin{align}
       \label{W(ell)}
       W^{(\ell )} := W_{\ell } \cdots W_1 \in \mathbb{R}^{d_{\ell }\times d_{0}},  
       \quad \text{ for } \ell=1,\cdots,L,
   \end{align}
   and
   \begin{align}
       \label{b(ell)}
       b^{(\ell )} := 
       \begin{cases}
           W_{\ell}\cdots W_2 b_1 + W_{\ell }\cdots W_3 b_2 + \cdots + W_{\ell } b_{\ell
           -1} + b_{\ell },  &\text{ if } \ell \geq 2, \\
           b_1, &\text{ if } \ell =1,
       \end{cases}
   \end{align}
   we have
   \begin{align} 
       \label{taurecursive}
       X^{(\ell)} &= W^{(\ell)}(X_0)^{(\tau, \ell)} + B^{(\ell)}, 
   \end{align}
   for $\ell =1, \cdots, L-1$, and
   \begin{align}
       \label{taulast}
       X^{(L)} &= W^{(L)} (X_0)^{(\tau, L-1)} + B^{(L)},
   \end{align}
   where 
   \begin{align*}
       (X_0)^{(\tau,\ell )}
           &:= \tau_{W^{(\ell)},b^{(\ell)}}(\tau_{W^{(\ell-1)},b^{(\ell-1)}}(\cdots
           \tau_{W^{(2)}, b^{(2)}}(\tau_{W^{(1)},b^{(1)}}(X^{(0)}))\cdots) \nonumber \\ 
           &= \tau_{W^{(\ell)},b^{(\ell)}} (X_0)^{(\tau ,\ell -1)}.
   \end{align*}
\end{prop}

\begin{proof}
    Using the fact that $W_{\ell} \pen{W_{\ell}}=\Id_{d_{\ell}\times d_{\ell}}$ for all
    $\ell =1,\cdots, L$, we have
    \begin{align*}
        X^{(\ell )} &=  \sigma(W_{\ell } X^{(\ell -1)} +B_{\ell }) \nonumber \\
                    &= W_{\ell }\pen{W_{\ell }} \sigma(W_{\ell } X^{(\ell -1)} +B_{\ell })
                     \nonumber \\
                    &= W_{\ell }\pen{W_{\ell }}\left(\sigma(W_{\ell } X^{(\ell -1)}
                    +B_{\ell }) - B_{\ell } \right) + B_{\ell }  \\
                    &= W_{\ell }\,  \tau_{W_{\ell }, B_{\ell }}(X^{(\ell -1)}) + B_{\ell
                    }, \nonumber
    \end{align*}
    which proves \eqref{tau_1}. Going one step further,
    \begin{align*}
        X^{(\ell )} &=  \sigma(W_{\ell } X^{(\ell -1)} +B_{\ell })  \nonumber \\
                    &= \sigma(W_{\ell }( W_{\ell -1}\, \tau_{W_{\ell -1}, b_{\ell
                    -1}}(X^{(\ell -2)}) + B_{\ell-1} ) + B_{\ell })  \nonumber \\
                    &\stackrel{(*)}{=} W_{\ell } W_{\ell -1} \pen{W_{\ell} W_{\ell -1}}
                    \sigma(W_{\ell }( W_{\ell -1}\, \tau_{W_{\ell -1}, b_{\ell
                    -1}}(X^{(\ell -2)}) + B_{\ell-1} ) + B_{\ell })  \\
                    &= W_{\ell } W_{\ell -1} \pen{W_{\ell} W_{\ell -1}}
                    \sigma(W_{\ell }W_{\ell -1}\, \tau_{W_{\ell -1}, b_{\ell
                    -1}}(X^{(\ell -2)}) + W_{\ell }B_{\ell-1}  + B_{\ell })  \nonumber \\
                    &= W_{\ell }W_{\ell -1} \, \tau_{W_{\ell }W_{\ell -1}, W_{\ell }b_{\ell
                    -1}+b_{\ell }} (\tau_{W_{\ell -1},b_{\ell -1}}(X^{(\ell -2)})) +
                    W_{\ell }B_{\ell -1}+B_{\ell }, \nonumber
    \end{align*}
    where $(*)$ follows from the surjectivity of $W_{\ell}W_{\ell-1}$.
    An iteration of these arguments then yields \eqref{taurecursive}. Finally,
    \begin{align*}
        X^{(L)} &= W_{L} X^{(L-1)} + B_{L}  \nonumber \\
                &= W_{L} (W^{(L-1)} X^{(\tau, L-1)} + B^{(L-1)}) + B_{L}  \\
                &=  W^{(L)} X^{(\tau, L-1)} + B^{(L)}. \nonumber
    \end{align*}
\end{proof}

This lemma, and especially expression \eqref{taulast}, suggest that we think of a neural
network with $L$ layers as a concatenation of $L-1$ truncation maps acting on the data in
input space $\mathbb{R}^{d_0}$, followed by an affine map to the appropriate
output space $\mathbb{R}^{d_{L}}$. These truncation maps and the affine map are given
by cumulative parameters, so that from now on in this paper a layer will be defined
by prescribing  $(W^{(\ell)}$, $b^{(\ell )})$ instead of $(W_{\ell }, b_{\ell })$, for all
$\ell =1\cdots, L$.

We will need the following properties of the cumulative parameters in the case where $d_0
\geq \cdots \geq d_{L}$, which applies to the results in Sections \ref{sec:clusters} and
\ref{sec:hyp}.

\begin{lemma}
    \label{penrose-concat}
    Consider a collection
     $\left\{W_{\ell }\right\}_{\ell =1}^{L}$ of linear maps which can be composed into
     surjective cumulative matrices $\{W^{(\ell)}\}_{\ell =1}^{L}$, defined as in
     \eqref{W(ell)}.
     Define
     \begin{align}
         \label{P(l)}
         P^{(\ell)} &:= \pen{W^{(\ell)}}W^{(\ell )}\, \pen{W^{(\ell -1)}}W^{(\ell
         -1)}\cdots \pen{W^{(1)}} W^{(1)} \nonumber \\
                    &= \pen{W^{(\ell )}}W^{(\ell )} P^{(\ell -1)}.
     \end{align}
     Then
     \begin{align} \label{P(l)=PenWW}
         P^{(\ell )}=\pen{W^{(\ell )}}W^{(\ell )},
     \end{align}
     and
     \begin{align} \label{WP=W}
         W^{(\ell )}P^{(\ell -1)}=W^{(\ell )}.
     \end{align}
\end{lemma}

\begin{proof}
    Note that $P^{(1)} = \pen{W^{(1)}}W^{(1)}$ by definition,
    and assume $P^{(\ell -1)} = \pen{W^{(\ell -1)}}W^{(\ell -1)}$. 
    If $W^{(\ell -1)}$ is surjective, then $W^{(\ell -1)}\pen{W^{(\ell -1)}} = \Id$ and
    \begin{align*}
        P^{(\ell)} &=\pen{W^{(\ell )}} W^{(\ell )}P^{(\ell -1)} \nonumber \\
                   &=\pen{W^{(\ell )}} W^{(\ell )}\pen{W^{(\ell-1)}}W^{(\ell -1)}
                   \nonumber \\
                   &=\pen{W^{(\ell )}} W_{\ell } \left( W^{(\ell -1)}\pen{W^{(\ell-1)}}
                   \right) W^{(\ell -1)} \\
                   &=\pen{W^{(\ell )}} W_{\ell } W^{(\ell -1)} = \pen{W^{(\ell
                   )}} W^{(\ell)}, \nonumber
    \end{align*}
    which gives \eqref{P(l)=PenWW}.

    We have proved that $P^{(\ell )}$ is the projector onto the orthogonal complement of
    the kernel of $W^{(\ell )}$
    for all $\ell =1,\cdots, L$. Now, let $Q^{(\ell -1)}=\Id - P^{(\ell -1)}$ be the
    projector to the kernel of $W^{(\ell -1)}$. Then
    \begin{align*}
        W^{(\ell )} (\Id-P^{(\ell -1)}) &= W^{(\ell )}Q^{(\ell -1)} \nonumber \\
                                        &= W_{\ell }W^{(\ell -1)}  Q^{(\ell -1)} \\
                                        &= 0, \nonumber
    \end{align*}
    and we have proved \eqref{WP=W}.
\end{proof}

\vspace{1ex}
The next lemma is the key tool used to prove the main results in Sections
\ref{sec:clusters} and \ref{sec:hyp}, as it describes the action of the truncation map in
some particular cases.
The activation function $\sigma$ acts in such a way that it is the identity on the
positive sector
\begin{align*}
    \mathbb{R}_{+}^{n} := \left\{(x_1,\cdots, x_n)^{T} \in \mathbb{R}^{n}: x_{i} \geq 0,
    i=1, \cdots, n \right\},
\end{align*}
and the zero map on the negative sector $\mathbb{R}_{-}^{n}$ (similarly defined).
Figure \ref{fig:truncation} shows the action of the truncation map in two dimensions,
for $W\in \mathbb{R}^{2\times 2}$ and $b\in \mathbb{R}^{2}$. In this simple context, it is
clear that for an invertible weight matrix, the action of  $\tau_{W,b}$ is completely
determined by a cone (given by the blue dashed lines on the first picture). While this is not
always true, it is still fruitful to consider a similar picture in general.
\begin{figure}[h]
    \centering
    \def\svgwidth{0.85\columnwidth}
    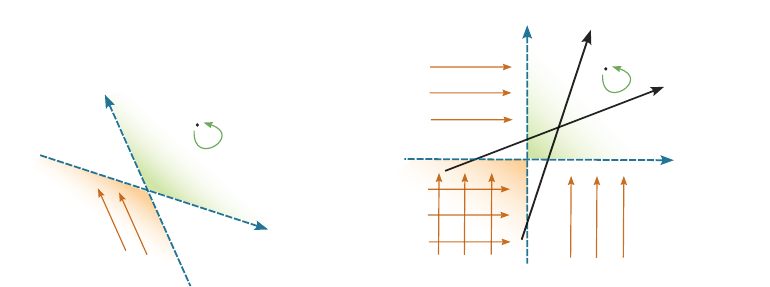
    \caption{The action of $\tau_{W,b}$ in two dimensions. The image on the left
    represents the input space given by coordinates $(x_1, x_2)$. On the right we see the
    image of this space under the map $x\mapsto \sigma(Wx + b)$, where $\sigma$ acts
    according to the new coordinate system $(y_1, y_2)$: it is the identity on the positive
    sector (shaded in green), and can be thought of as projecting along the orange arrows
    in the other regions, which gives the total effect of mapping the negative sector
    (shaded in orange) to the origin. Pulling this action back to input space produces the
    cone on the picture to the left.} 
    \label{fig:truncation}
\end{figure}

We define a cone of angle $\theta \leq \pi$ around a given vector $h\in \mathbb{R}^{n}$
and based at the origin,
\begin{align}
    \label{cone}
    \mathfrak{C}_{\theta}[h] := \left\{x \in \mathbb{R}^{n}: \angle (x, h) \leq
    \frac{\theta }{2}\right\}.
\end{align}
When the forward direction $h$ is clear, we will refer to a pair
$\mathfrak{C}_{\theta}[h]$, $\mathfrak{C}_{\theta}[-h]$ as the forward and backward cones,
respectively.
Then, for 
\begin{align}
    \label{thetan}
    \theta_{n} := 2 \arccos\frac{\sqrt{n-1}}{\sqrt{n}},
\end{align}
we note that $\mathfrak{C}_{\theta_{n}}[u_{n}] \subset \mathbb{R}^{n}$ is the largest
cone centered around the diagonal $u_{n} = (1, \cdots, 1)^{T}$ that is contained fully in
$\mathbb{R}_{+}^{n}$.\footnote{See appendix.} Similarly, the backward cone is fully
contained in $\mathbb{R}_{-}^{n}$. 
A cone $\mathfrak{C}_{\theta}[u_{n}]$ of larger angle can be made to fit into
$\mathbb{R}_{+}^{n}$ by a transformation which fixes the axis $u_n$ and shrinks
every other direction perpendicular to the axis according to 
    \begin{align}
        \label{lambda}
        \lambda(\theta,\theta_{n}) :=
        \begin{cases}
            \frac{\tan\frac{\theta_{n}}{2}}{\tan\frac{\theta}{2}} <1, &\text{ if }
            \theta > \theta_{n}, \\ 
            1, &\text{ otherwise.}
        \end{cases}
    \end{align}

Now, in the case where there is no change in dimension, Lemma \ref{lemma}
below states that for any given cone we can construct $W$ and $b$ so that the truncation
map $\tau_{W,b}$ acts as the identity on the forward cone, and projects the entirety of
the backward cone to its base point. However, if $\tau_{W,b}$ is encoding a layer in
a network for which there is a reduction in dimension, its image must also undergo a
dimensional reduction.  To account for the possible decrease in width of the network,
consider the family of projectors
\begin{align} \label{proj}
    \P nm&:\mathbb{R}^{n} \to \mathbb{R}^{m},\quad n\geq m, \nonumber\\
    \P nm(e^{n}_{i}) &=
    \begin{cases}
        e^{m}_{i}, \quad i = 1, \cdots, m, \\
        0, \quad \text{else,}
    \end{cases}
\end{align}
where $\left\{e_{i}^{n}\right\}_{i=1}^{n}$ is the standard basis for $\mathbb{R}^{n}$,
so that the matrix $\P nm \in \mathbb{R}^{m \times n}$ is a block matrix consisting of an
identity matrix $\Id_{m\times m}$ and zeroes.

\begin{lemma}
    \label{lemma}
    Consider the cone $p + \mathfrak{C}_{\theta}[h] \subset \mathbb{R}^{n}$ centered
    around a unit vector $h\in \mathbb{R}^{n}$ and based at $p \in \mathbb{R}^{n}$,
    and consider $W: \mathbb{R}^{n} \to \mathbb{R}^{m},\, b \in \mathbb{R}^{m}$ such that 
    \begin{align*}
        W &= \P nm W_{\theta} R, \nonumber \\
        b &= -Wp ,
    \end{align*}
    where
    \begin{align}
         \label{Wtheta}
        W_{\theta} = \tilde{R} \, \diag(1, \lambda(\theta,\theta_{n}), \cdots,
        \lambda(\theta,\theta_{n})) \, \tilde{R}^{T} \in \mathbb{R}^{n \times n},
    \end{align}
    with $\lambda(\theta,\theta_{n})$ defined as in \eqref{lambda}
    and $R, \tilde{R} \in \SO(n, \mathbb{R})$ are such that 
    \begin{align*}
        R h = \frac{u_{n}}{\abs{u_{n}}}, \,\, \tilde{R} e^{n}_1 =
    \frac{u_{n}}{\abs{u_{n}}}.  
    \end{align*}

    Then 
    \begin{align*}
       \tau_{W,b}(x) = 
       \begin{cases}
           (\pen{W}W)\, x, \quad  x \in p + \mathfrak{C}_{\theta}[h], \\
           (\pen{W}W)\, p, \quad  x \in p + \mathfrak{C}_{\theta}[-h].
       \end{cases}
   \end{align*}

    Moreover, if $P:\mathbb{R}^{n}\to \mathbb{R}^{n}$ is a map such that $WP =W$, then
    \begin{align}
        \label{tauP}
        \tau_{W,b}(Px) = \tau_{W,b}(x).
    \end{align}
\end{lemma}

\begin{proof}
    First, consider $x\in p + \mathfrak{C}_{\theta}[h]$. Since
    \begin{align*}
        \tau_{W,b}(x) = \pen{W} \left( \sigma( Wx + b) -b \right),
    \end{align*}
   it suffices to show that $Wx+b \in \mathbb{R}_{+}^{n}$, and so $\sigma(Wx + b) = Wx +
   b$. Substituting the given expressions for $W$ and $b$,
   \begin{align*}
       Wx + b &= W(x - p) = \P nm W_{\theta} R(x-p),
   \end{align*}
   then $x \in p + \mathfrak{C}_{\theta}[h] \text{ implies } (x - p) \in
   \mathfrak{C}_{\theta}[h]$, and we observe the action of the maps on the cones. First,
   an orthogonal transformation changes the axis, but not the opening angle, of a cone,
   and we have
   \begin{align*}
       R(x-p) \in \mathfrak{C}_{\theta}[Rh] = \mathfrak{C}_{\theta}[u_{n}].
   \end{align*}
   Then, the map
   $W_{\theta}$ is constructed so that it fixes the axis, but changes the aperture of the
   cone so that $W_{\theta} \mathfrak{C}_{\theta}[u_{n}] \subset
   \mathfrak{C}_{\theta_{n}}[u_{n}]$, which is contained entirely in $\mathbb{R}_{+}^{n}$,
   so
   \begin{align*}
       W_{\theta}R ( \mathfrak{C}_{\theta}[h]) \subset \mathfrak{C}_{\theta_{n}}[u_{n}] \subset
       \mathbb{R}_{+}^{n},
   \end{align*}
   and $W(x-p) \in \mathbb{R}_{+}^{m}$.

   For the backward cone, $x \in p + \mathfrak{C}_{\theta}[-h]$, similar arguments yield
   \begin{align*}
       W( \mathfrak{C}_{\theta}[-h]) \subset \mathfrak{C}_{\theta_{n}}[-u_{n}] \subset
       \mathbb{R}_{-}^{n},
   \end{align*}
   and then $W(x-p) \in \mathbb{R}_{-}^{m}$ implies $\sigma(Wx + b) = 0$ and
   \begin{align*}
       \tau_{W,b}(x) = \pen{W}(-b) = \pen{W}W p.
   \end{align*}

   Finally, \eqref{tauP} follows easily from the definition of the truncation map,
   \begin{align*}
       \tau_{W,b}(Px) &= \pen{W} \left(\sigma(W(Px) + b) - b \right) \nonumber \\
                      &=\pen{W} \left(\sigma(Wx + b) - b \right)\\  
                      &= \tau_{W,b}(x). \nonumber
   \end{align*}

\end{proof}

We can now use the truncation map and Lemma \ref{lemma} to generalize the construction of
zero loss global minima \citep{chenewald23deep} in two different directions.  First,
in Section \ref{sec:clusters}, we keep the training data as sufficiently separated, 
small enough clusters, but allow for $M = d_0 \geq d_1 \geq \cdots \geq d_L = Q$.
Then, in Section \ref{sec:hyp}, we hold the width of every layer but the last one constant,
$d_0 = d_1 = \cdots = d_{L-1}$, but allow for $M>Q$ and for a more general distribution of
the training data.

\begin{remark}
Note that the decomposition
\begin{align*}
    W = W_{\theta} R
\end{align*}
coincides with the polar decomposition of a matrix
\begin{align*}
    W = \abs{W} R
\end{align*}
for $\abs{W} := (W W^{T})^{\frac{1}{2}}$ and $R\in O(n, \mathbb{R})$. Indeed, 
\begin{align*}
    W W^{T} &= \tilde{R} \diag(1, \cdots) \tilde{R}^{T} R R^{T} \tilde{R} \diag(1, \cdots)
    \tilde{R}^{T} \nonumber \\
            &= \tilde{R}\diag(1,\cdots)^{2} \tilde{R}^{T} \\
            &= W_{\theta}^{2}.\nonumber 
\end{align*}
\end{remark}

\section{Clustered Data and Dimensional Reduction} \label{sec:clusters}

We will consider in this section the same geometry on the data as that adopted by
\citet{chenewald23deep}, which we now briefly recall. Let
\begin{align*}
    \overline{x} := \frac{1}{Q} \sum_{j=1}^{Q} \overline{x_{0,j}}
\end{align*}
be the average of all class means, and note that this is not necessarily the average of
all data points. 
We will assume that
\begin{align}
    \label{clustercond1}
    \delta := \sup_{i,j} \abs{\Delta x_{0,j,i}} 
    < c_0 \min_{j} \abs{\overline{x_{0,j}} - \overline{x}},
\end{align}
for $c_0 \in (0, \frac{1}{4})$.
Moreover, consider
\begin{align} 
    \label{fj}
    f_{j} := \frac{\overline{x} - \overline{x_{0,j}}}{\abs{\overline{x} -
    \overline{x_{0,j}}}}, \quad j=1, \cdots, Q,
\end{align}
the unit vectors pointing from $\overline{x_{0,j}}$ to $\overline{x}$. 
For each $j=1,\cdots, Q$, let $\theta_{*,j}$ be the smallest angle such that
\begin{align*}
    \bigcup_{j'\neq j} B_{4 \delta}( \overline{x_{0,j'}}) \subset \overline{x_{0,j}} +
    \mathfrak{C}_{\theta_{*,j}}[f_{j}],
\end{align*}
and assume 
\begin{align}
    \label{clustercond2}
    \max_{j} \theta_{*,j} < \pi.
\end{align}

\begin{prop}
    \label{decreasingwidths}
    Consider a set of training data $\mathcal{X}_{0} = \bigcup_{j=1}^{Q} \mathcal{X}_{0,j}
    \subset \mathbb{R}^{M}$ separated into $Q$ classes corresponding to linearly
    independent labels $\{y_{j}\}_{j=1}^{Q} \subset \mathbb{R}^{Q}$.
    If conditions \eqref{clustercond1} and \eqref{clustercond2} are satisfied, then a
    neural network with ReLU activation function defined as in \eqref{neuralnet}, with
    $L = Q+1$ layers, $d_0=M$, $d_{Q+1} = Q$, and $d_{\ell -1} \geq d_{\ell } \geq Q$ for all
    $\ell =1, \cdots, L$, attains 
    \begin{align*}
        \min_{(W_{i}, b_{i})_{i=1}^{L}} \mathcal{C}[(W_{i}, b_{i})_{i=1}^{L}] 
        = 0,
    \end{align*}
    with weights and biases given recursively by \eqref{W(ell)}, \eqref{b(ell)} and the
    following cumulative parameters: for $\ell =1, \cdots, Q$, 
    \begin{align*}
        W^{(\ell)} &= \P{d_0}{d_{\ell}} W_{\theta_{*}} R_{\ell}, \nonumber \\
        b^{(\ell)} &= -W^{(\ell)}( \overline{x_{0,\ell}} + \mu_{\ell } f_{\ell }),
    \end{align*}
    for $\P{d_0}{d_{\ell }} \in \mathbb{R}^{d_{\ell }\times d_0}$ a projection,
    $W_{\theta_{*}}\in GL(d_0, \mathbb{R})$ and $R_{\ell }\in \SO(d_0, \mathbb{R})$
    constructed as in Section \ref{sec:truncones},
    and some  $\mu_{\ell } \in (2 \delta, 3 \delta )$;
    for $\ell = Q+1$,
    \begin{align*}
        W^{(Q+1)} &= Y \pen{\overline{X_0^{red}}^{(\tau,Q)}}, \nonumber \\
        b^{(Q+1)} &= 0.
    \end{align*}
    This minimum is degenerate.
\end{prop}

Since the only constraints on the architecture are that the widths of the layers must
be non-increasing and bounded below by $Q$, one could construct a network for which the
first hidden layer immediately reduces to $Q$ dimensions, and the following layers have
constant width. This leads to a reduction in the total number of parameters needed, and we
have proved

\begin{coro}
    For training data $\mathcal{X}_{0} = \bigcup_{j=1}^{Q} \mathcal{X}_{0,j} \subset 
    \mathbb{R}^{M}$ satisfying the same conditions as above, there exists a neural
    network with $Q(M + Q^{2})$ parameters interpolating the data, and this does not
    depend on $\abs{\mathcal{X}_{0}} = N$.
\end{coro}

Thus, well-distributed data makes it possible to attain zero loss even in the
under-parametrized setting.

We defer the full proof to Section \ref{sec:proof}.

\section{Hyperplanes and Sequentially Separable Data} \label{sec:hyp}

In the previous section, we assumed that the data had to be neatly contained in
sufficiently small
and well spaced balls. However, when using Lemma \ref{lemma} in the proof, it was
sufficient to check that each ball lay inside a certain backward or forward cone. More
generally, for
the action of any given layer described by a truncation map, all that matters is that
there exists a cone such that the class that is being singled out sits entirely within the
backward cone, and the remaining data sits in the forward cone. 

From this point of view, the action of each iteration of the truncation map can be thought
of as performing a one-versus-all classification: a particular class is being singled out
and separated from the remaining classes. However, what we obtain is different from simply
breaking down a multiclass classification problem into several binary classification
tasks. Data that has been truncated by a layer will remain truncated, and so the order in
which the one-versus-all tasks are done is important.
To account for this, we extend the notion of linear separability from binary
classification to a problem with multiple classes as follows.

\begin{defi}
    \label{sls}
    We say a data set is \emph{sequentially linearly separable} if there exists an
    ordering of the classes such that for each $j = 1, \cdots, Q$, there exists 
    a hyperplane $H_{j}$ and a point $p_{j} \in H_{j}$ such that $H_{j}$
    separates $\mathcal{X}_{0,j}$
    and $\bigcup_{k<j} \{p_{k}\} \cup
    \bigcup_{k>j} \mathcal{X}_{0,k}$.\footnote{See Figure \ref{fig:sls} in the
    introduction for an example.}
\end{defi}

Let $H(p, \nu)$ be the hyperplane defined by a point $p$ and a normal vector $\nu$,
\begin{align*}
    H(p,\nu) := \left\{p+x \in \mathbb{R}^{n}: \langle x,\nu\rangle = 0 \right\}.
\end{align*}
For data that is sequentially linearly separable, we have 
the corresponding hyperplanes $H_{j} = H(p_{j}, h_{j})$ for some unit vectors $h_{j}$,
$j=1, \cdots, Q$. We will always consider data sets to be finite, and so the distance from
any such hyperplane to the each of the two sets it is separating is always positive.
Moreover, a pair $(p,\nu) \subset \mathbb{R}^{n}\times \mathbb{R}^{n}$, along with an
angle $\theta <\pi$, define a pair of cones $p + \mathfrak{C}_{\theta}[\nu]$ and $p+
\mathfrak{C}_{\theta}[-\nu]$ as in \eqref{cone}, and so to each hyperplane $H(p,\nu)$ is
associated a family of cones.

\begin{remark}
    \label{rmk-li}
    Similarly, we note that there exists $\epsilon >0 $ such that each point
    $p_{j}$ can be freely perturbed on an $\epsilon$-ball around it (possibly moving the
    corresponding hyperplane $H_{j}$ as well) so that $(\tilde{p}_{j},
    \tilde{H}_{j})_{j=1}^{Q}$ still satisfies Definition \ref{sls}, for $\tilde{p}_{j} \in
    (p_{j} + B_{\epsilon}(p_{j})) \cap \tilde{H}_{j}$, $j=1, \cdots, Q$.
    Therefore, we may assume that $\{p_{j}\}_{j=1}^{Q}$ form a linearly independent set.
\end{remark}

\begin{thm} 
    \label{mainthm}
    Consider a set of training data $\mathcal{X}_{0} = \bigcup_{j=1}^{Q} \mathcal{X}_{0,j}
    \subset \mathbb{R}^{M}$ separated into $Q$ classes corresponding to linearly
    independent
    labels $\{y_{j}\}_{j=1}^{Q} \subset \mathbb{R}^{Q}$.
    If the data is sequentially linearly separable, then a
    neural network with ReLU activation function defined as in \eqref{neuralnet}, with
    $L = Q+1$ layers, $d_0=M$, $d_{Q+1} = Q$, and $d_0 = d_{\ell } \geq Q$ for all hidden
    layers $\ell =1, \cdots, Q$, attains 
    \begin{align*}
        \min_{(W_{i}, b_{i})_{i=1}^{L}} \mathcal{C}[(W_{i}, b_{i})_{i=1}^{L}] 
        = 0,
    \end{align*}
    with weights and biases given recursively by \eqref{W(ell)}, \eqref{b(ell)} for the
    following cumulative parameters: for $\ell = 1, \cdots, Q$, 
    \begin{align*}
        W^{(\ell)} &=  W_{\theta_{\ell}} R_{\ell}, \nonumber \\
        b^{(\ell)} &= -W^{(\ell)}p_{\ell},
    \end{align*}
    for $W_{\theta_{\ell}}\in GL(d_0, \mathbb{R})$ and
    $R_{\ell }\in \SO(d_0, \mathbb{R})$ constructed as in Section \ref{sec:truncones}; and
    for $\ell = Q+1$,
    \begin{align*}
        W^{(Q+1)} &= Y \pen{\overline{X_0^{red}}^{(\tau,Q)}}, \nonumber \\
        b^{(Q+1)} &= 0.
    \end{align*}
    This minimum is degenerate.
\end{thm}

\begin{proof}
    For $j=1, \cdots, Q$, let
    \begin{align}
        \label{pjhj}
        H_{j} = H(p_{j}, h_{j})
    \end{align}
    be a sequence of hyperplanes that realizes the sequential linear separability of the
    given data, and assume the ordering is correct as given.
    As $\mathcal{X}_{0}$ is a finite set, for each $j$ there exists $\theta_{j,min} < \pi$
    such that any $\theta_{j} \in (\theta_{j,min}, \pi)$ implies 
    \begin{align}
        \label{thetaj}
        p_{j} + \mathfrak{C}_{\theta_{j}}[h_{j}] &\supset  \nonumber
        \bigcup_{k<j} \{p_{k}\} \cup \bigcup_{k>j} \mathcal{X}_{0,k} \\  
        \text{ and } \\
        p_{j} + \mathfrak{C}_{\theta_{j}}[- h_{j}] &\supset \mathcal{X}_{0,j}. \nonumber
    \end{align}
    We will prove that $X_{0,j}^{(\tau, Q)} = p_{j} u_{N_{j}}^{T}$, or that for all
    $x\in \mathcal{X}_{0,j}$, $x^{(\tau , Q)} = p_{j}$, for  $j=1,\cdots, Q.$ 
    
    \vspace{1ex}
    \textbf{Induction base \texorpdfstring{$\ell =1$}{l=1}.} 
    Given $p_1, h_1,$ and $\theta_{1} \in (\theta_{1,min}, \pi)$, we have
    \begin{align}
        \label{cones1}
        \mathcal{X}_{0,1} &\subset p_1 + \mathfrak{C}_{\theta_{1}}[-h_1], \nonumber \\ 
        \bigcup_{j>1} \mathcal{X}_{0,j} &\subset p_1 + \mathfrak{C}_{\theta_{1}}[h_1].
    \end{align}
    Let
    \begin{align*}
        W_1 = W^{(1)} := W_{\theta_{1}} R_1 \in \mathbb{R}^{d_0\times d_0},
    \end{align*}
    for $W_{\theta_{1}}$ defined as in \eqref{Wtheta}, and $R_1\in \SO(d_0,\mathbb{R})$
    such that 
    \begin{align*}
        R_1 h_1 = \frac{u_{d_0}}{\abs{u_{d_0}}},
    \end{align*}
    and
    \begin{align*}
        b_1 = b^{(1)} := -W_1 p_1.
    \end{align*}
    Then Lemma \ref{lemma} implies we have the following truncation map
    \begin{align*}
       \tau _{W^{(1)}, b^{(1)}}(x) =
       \begin{cases}
           x, \quad &x \in p_1 + \mathfrak{C}_{\theta_{1}}[h_1], \\
           p_1, \quad &x \in p_1 + \mathfrak{C}_{\theta_{1}}[-h_1],
       \end{cases}
    \end{align*}
    which, according to \eqref{cones1}, acts on the data as
    \begin{align*}
       X_{0,j}^{(\tau ,1)} = \tau _{W^{(1)}, b^{(1)}}(X_{0,j}) =
       \begin{cases}
           p_1 u_{N_1}^{T}, \quad &j=1, \\
           X_{0,j}, \quad &j >1,
       \end{cases}
    \end{align*}
    so that the first layer has successfully shrunk the first equivalence class of data
    points $\mathcal{X}_{0,1}$ to a single point.
    
    \vspace{1ex}
    \textbf{The induction step \texorpdfstring{$\ell -1 \to \ell$}{l-1 -> l}.}
    Assume 
    \begin{align}
        \label{ind-hyp}
        X_{0,j}^{(\tau, \,\ell -1)} = 
        \begin{cases}
            p_{j} u_{N_{j}}^{T}, \quad &j \leq \ell -1, \\
            X_{0,j}, \quad &j > \ell -1,
        \end{cases}
    \end{align}
    where $p_{j}$ are the points given by \eqref{pjhj}. We will construct cumulative
    parameters $W^{(\ell )}, b^{(\ell )}$ such that
    \begin{align}
       \label{ind-concl}
       X_{0,j}^{(\tau ,\ell)} = \tau_{W^{(\ell )}, b^{(\ell )}} (X_{0,j}^{(\tau, \ell-1)}) = 
       \begin{cases}
           p_{j} u_{N_{j}}^{T} = X_{0,j}^{(\tau, \ell -1)}, \quad &j < \ell , \\
           p_{\ell} u_{N_{\ell}}^{T}, \quad &j=\ell, \\
           X_{0,j} = X_{0,j}^{(\tau, \ell -1)}, \quad &j > \ell.
       \end{cases}
    \end{align}
    By sequential linear separability of the data, there exists a hyperplane $H_{\ell} =
    H(p_{\ell},h_{\ell })$ in
    the collection given by \eqref{pjhj}
    separating $\mathcal{X}_{0,\ell}$ and $\bigcup_{j<\ell} \{p_{j}\} \cup \bigcup_{j>\ell }
    \mathcal{X}_{0,j}$.
    We can use the induction hypothesis \eqref{ind-hyp} to rewrite 
    \begin{align*}
        \mathcal{X}_{0,\ell } = \mathcal{X}_{0,\ell }^{(\tau ,\ell -1)} \,\text{ and }\,
        \bigcup_{j<\ell} \{p_{j}\} \cup \bigcup_{j>\ell } \mathcal{X}_{0,j} = 
        \bigcup_{j\neq \ell } \mathcal{X}_{0,j}^{(\tau , \ell -1)},
    \end{align*}
    and so $H_{\ell }$ separates $\mathcal{X}_{0,\ell }^{(\tau ,\ell-1)}$ and
    $\bigcup_{j\neq \ell } \mathcal{X}_{0,j}^{(\tau , \ell -1)}$.
    
    Then by \eqref{thetaj}, there is a $\theta_{\ell} < \pi$ such that
    \begin{align}
        \label{ellcones}
        \bigcup_{j\neq \ell } \mathcal{X}_{0,j}^{(\tau , \ell -1)} \subset
        p_{\ell} + \mathfrak{C}_{\theta_{\ell }}[h_{\ell }]  
        \quad \text{ and } \quad
        \mathcal{X}_{0,\ell } \subset
        p_{\ell } + \mathfrak{C}_{\theta_{\ell }}[- h_{\ell }].
    \end{align}
    Now we can choose 
    \begin{align}
        \label{Wellproof}
        W^{(\ell )} := W_{\theta_{\ell}} R_{\ell} \in \mathbb{R}^{d_{0} \times d_0},
        \nonumber \\ 
        b^{(\ell )} := - W^{(\ell )}\,  p_{\ell } \in \mathbb{R}^{d_{0}},
    \end{align}
    for $W_{\theta_{\ell }}$ defined as in \eqref{Wtheta} and $R_{\ell }\in \SO(d_{0},
    \mathbb{R})$ such that
    \begin{align*}
        R_{\ell } h_{\ell } = \frac{u_{d_0}}{\abs{u_{d_0}}},
    \end{align*}
    so that Lemma \ref{lemma} gives
    \begin{align*}
       \tau _{W^{(\ell)}, b^{(\ell)}}(x) =
       \begin{cases}
           x, \quad &x \in p_{\ell} + \mathfrak{C}_{\theta_{\ell}}[h_{\ell}], \\
           p_{\ell}, \quad &x \in p_{\ell} + \mathfrak{C}_{\theta_{\ell}}[-h_{\ell}],
       \end{cases}
    \end{align*}
    which, along with \eqref{ellcones}, implies \eqref{ind-concl}.
    
    \vspace{1ex}
    \textbf{Conclusion of proof.} 
    At the end of $Q$ layers, or $Q$ applications of the truncation maps $\tau
    _{W^{(\ell)}, b^{(\ell)}}$, we have mapped each equivalence class $\mathcal{X}_{0,j}$ 
    to a single point, $(\mathcal{X}_{0,j})^{(\tau, Q)} = p_{j}$, for all $j=1, \cdots,
    Q$. This implies that $(\Delta X_0)^{(\tau, Q)} =0$.

    In the last layer, we have
    \begin{align*}
        X^{(Q+1)} &= W^{(Q+1)}(X_0)^{(\tau, Q)} + B^{(Q+1)},
    \end{align*}
    and we wish to minimize \footnote{Recall the matrices defined in \eqref{X_0red}.}
    \begin{align*}
        \min_{W^{(Q+1)}, b^{(Q+1)}} \norm{W^{(Q+1)}(X_0)^{(\tau, Q)} + B^{(Q+1)} -
        Y^{ext}}^{2}_{\mathcal{L}^{2}}.
    \end{align*}
    This can be achieved by letting 
    \begin{align*}
        b^{(Q+1)} := 0,
    \end{align*}
    and solving
    \begin{align*}
        W^{(Q+1)} (X_0)^{(\tau , Q)} = Y^{ext}.
    \end{align*}
    Since  
    \begin{align*}
        (X_0)^{(\tau, Q)} &= ( \overline{X_0} + \Delta X_0)^{(\tau , Q)} \nonumber \\
                          &= ( \overline{X_0})^{(\tau ,Q)} + (\Delta X_0)^{(\tau ,
                          Q)} \\
                          &= ( \overline{X_0})^{(\tau ,Q)},\nonumber 
    \end{align*}
    we can equivalently solve
    \begin{align}
        \label{leastsq}
        W^{(Q+1)} (\overline{X_0^{red}})^{(\tau , Q)} = Y.
    \end{align}
    As the collection $\{p_{j}\}_{j=1}^{Q}$ is 
    assumed linearly independent (see Remark \ref{rmk-li}),  
    \begin{align*}
        (\overline{X_0^{red}})^{(\tau , Q)} = \left[ p_1 \cdots p_{Q}\right]
    \end{align*}
    satisfies
    \begin{align*}
        \pen{(\overline{X_0^{red}})^{(\tau , Q)}} (\overline{X_0^{red}})^{(\tau , Q)} =
        \Id_{Q\times Q}, 
    \end{align*}
    and so
    \begin{align}
        \label{WQ+1proof}
        W^{(Q+1)} = Y \pen{(\overline{X_0^{red}})^{(\tau , Q)}}
    \end{align}
    is a solution to \eqref{leastsq}, and we have in fact produced a collection of weights
    and biases $(W_{i}^{**}, b_{i}^{**})_{i=1}^{Q+1}$ given by the cumulative parameters
    defined in \eqref{Wellproof} and \eqref{WQ+1proof} that achieves zero cost,
    \begin{align*}
        \mathcal{C}[(W_{i}^{**}, b_{i}^{**})_{i=1}^{Q+1}] = 0.
    \end{align*}

    \vspace{1ex}
   Finally, the minimum we have constructed is degenerate: for a fixed collection of
   hyperplanes, the angles $(\theta_{\ell})_{\ell =1}^{Q}$ of the associated cones can
   vary in $(\theta_{1,min}, \pi)\times \cdots \times (\theta_{Q,min}, \pi)$, and so
   parametrize a family of cumulative weights $(W^{(\ell)}[\theta_{\ell }])_{\ell=1}^{Q}$.  
\end{proof}

\section{Additional Remarks} \label{sec:rmk}

We make here a few remarks concerning generalization and neural collapse.

\subsection{Generalization} \label{sec:gen}

Let $ \rho \in \mathbb{P}(Z)$ be a probability distribution on
$Z = X \times Y$, the set of data $x \in X$ and labels $y\in Y \subset
\mathbb{R}$. Following the most general setup in the book by \citet{vapnik}, we consider 
\begin{align*}
    \left\{Q(z, \alpha ), \alpha \in \Lambda \right\}
\end{align*}
to be a parametrized set of functions representing the composition of a loss function with
a function $f_{\alpha}:X \to Y$.  
The real risk of a function $f_{\alpha}$ 
on data distributed according to $\rho$ is
\begin{align*}
    R(\alpha) := \int_{Z} Q(z, \alpha ) \dif \rho(z).
\end{align*}
Given a training set $\left\{z_1, \cdots, z_{n}\right\}$, the empirical risk is 
\begin{align*}
    R_{emp}(\alpha) := \frac{1}{n} \sum_{i=1}^{n} Q(z_{i}, \alpha ).
\end{align*}
Then if $a \leq Q(z, \alpha ) \leq b$, the following holds
for any $\alpha \in \Lambda$ and with probability at least $1- \eta$:
\begin{align*}
    R(\alpha ) \leq R_{emp}(\alpha ) + (b-a)\sqrt{\frac{VC(\Lambda)\log(n+1) -
    \log(\frac{\eta}{4})}{n}},
\end{align*}
where $VC(\Lambda)$ is the Vapnik--Chervonenkis dimension of the function class
parametrized by $\Lambda$.
In particular, for an empirical risk minimizer which achieves zero loss,
\begin{align}
    \label{genbound}
    R(\alpha_{*}) \leq  (b-a)\sqrt{\frac{VC(\Lambda)\log(n+1) -
    \log(\frac{\eta}{4})}{n}}.
\end{align}

In this work, $X = \mathbb{R}^{d_0}$, $Y = \mathbb{R}^{Q}$, the functions
$f_{\alpha}:\mathbb{R}^{d_0} \to \mathbb{R}^{Q}$ are ReLU neural networks parametrized by
weights and biases, and we consider
the squared loss, so
\begin{align*}
    Q(z = (x, y), \alpha ) = \abs{f_{\alpha}(x) - y}^{2}.
\end{align*}
The loss function can be bounded, for instance, by imposing bounds on
the norms of the parameters and the input data. See work by \citet{chenetal25} for a more
detailed discussion.  

If the VC dimension of the function class is finite,\footnote{It is known that the VC
    dimension of real valued ReLU neural networks is finite, and explicit bounds exist \citep[see e.g.,][]{bartlettetal19}.} then \eqref{genbound} implies that the results obtained
in the work at hand generalize well to test data which is identically distributed to the
training data, as they do not depend on the size of the training data set.

\subsection{Cost Decomposition and Neural Collapse} \label{sec:nc}

Note that since the constructions in this paper do not depend on the cost function chosen,
we can change it to average over each class separately,\footnote{Equivalently, we could
consider balanced classes.}
\begin{align*}
    \mathcal{C}_{\mathcal{N}}[(W_{i},b_{i})_{i=1}^{L}] &:= \sum_{j=1}^{Q}
    \frac{1}{N_{j}} \sum_{i=1}^{N_{j}} \abs{x_{0,j,i}^{(L)} - y_{j}}^{2} 
                                                    \nonumber \\ &=
    \norm{X^{(L)}-Y^{ext}}^{2}_{\mathcal{L}_{\mathcal{N}}^{2}},
\end{align*}
where $\mathcal{N}:= \diag(N_{j} \Id_{N_{j}\times N_{j}}| j=1, \cdots, Q)$ and 
$\langle A,B\rangle_{\mathcal{L}_{\mathcal{N}}^{2}} := \Tr(A \mathcal{N}^{-1}
    B^{T})$.
For simplicity, throughout this section we will refer to $\mathcal{C}_{\mathcal{N}}$ as
$\mathcal{C}$. This can be decomposed as follows
\begin{align}
    \label{cost}
    \sum_{j=1}^{Q} \frac{1}{N_{j}} \sum_{i=1}^{N_{j}} \abs{x_{0,j,i}^{(L)} -
    y_{j}}^{2} &= 
    \sum_{j=1}^{Q} \frac{1}{N_{j}} \sum_{i=1}^{N_{j}} \left( \abs{x_{0,j,i}^{(L)}}^{2}
    - 2 \langle x_{0,j,i}^{(L)},y_{j}\rangle + \abs{y_{j}}^{2} \right) \nonumber \\ 
               &\hspace{-37pt}= \sum_{j=1}^{Q} \left( \frac{1}{N_{j}} \sum_{i=1}^{N_{j}}
               \abs{x_{0,j,i}^{(L)}}^{2} - \abs{\overline{x_{0,j}^{(L)}}}^{2} +
               \abs{\overline{x_{0,j}^{(L)}}}^{2} - 2 \langle
               \frac{1}{N_{j}} \sum_{i=1}^{N_{j}} x_{0,j,i}^{(L)},y_{j}\rangle +
           \frac{N_{j}}{N_{j}} \abs{y_{j}}^{2} \right) \nonumber \\
    &\hspace{-37pt}= \sum_{j=1}^{Q}  \left( \frac{1}{N_{j}} \sum_{i=1}^{N_{j}}
               \abs{x_{0,j,i}^{(L)} - \overline{x_{0,j}^{(L)}}}^{2} +
               \abs{\overline{x_{0,j}^{(L)}} - y_{j}}^{2}  \right)  \\
    &\hspace{-37pt}= \sum_{j=1}^{Q}  \frac{1}{N_{j}} \sum_{i=1}^{N_{j}}
               \abs{\Delta x_{0,j,i}^{(L)}}^{2} + \sum_{j=1}^{Q}
               \abs{\overline{x_{0,j}^{(L)}} - y_{j}}^{2} \nonumber  \\
    &\hspace{-37pt}= \sum_{j=1}^{Q}  \frac{1}{N_{j}} \sum_{i=1}^{N_{j}}
               \abs{W^{(L)} \Delta x_{0,j,i}^{(\tau, L-1)}}^{2} + \sum_{j=1}^{Q}
               \abs{W^{(L)} \overline{x_{0,j}^{(\tau, L-1)}} + b^{(L)} - y_{j}}^{2}.
                \nonumber
\end{align}
Since both terms are non-negative, the cost will be zero if, and only if, each
term is zero, and so we
have the following interpretation for the results presented in this paper: the
truncation maps implemented by the hidden
layers have the role of minimizing the first term, by which all variances are reduced
to zero, while the last (affine) layer
minimizes the second term, whereby the class averages are matched with the reference
outputs $y_{j}$.

\vspace{1em}
\citet*{hanpapyandonoho22} observed that when training a classifier to zero
mean square loss, the following phenomena are observed on the penultimate layer features:
\begin{itemize}
    \item (NC1) Within-class variability collapse, which for instance occurs if $\Delta
        X_{0}^{(L-1)} \to 0$.
\item (NC2) Convergence to simplex ETF: the vectors $(\overline{x_{0,j}^{(L-1)}} -
    \overline{x^{(L-1)}})$ evolve to have maximal angle between them, and the same
        length, forming an equilateral simplex, or simplex equiangular tight
        frame (ETF).
    \item (NC3) Convergence to self-duality: $f_{j}^{(L-1)} + w_{j} \to 0$, where
    $f_{j}^{(L-1)}$ are unit vectors pointing from $\overline{x_{0,j}^{(L-1)}}$ to
    $\overline{x^{(L-1)}}$, defined analogously to \eqref{fj}, and $w_{j}$ are the
        rows of $W_{L}$, for $j=1, \cdots, Q$.
\end{itemize}
They consider balanced classes and decompose the mean square loss into the sum of a
least-squares term and the remainder, and then further decompose the least-squares term
into terms associated to the properties above:
\begin{align}
    \label{NCcost}
    \mathcal{C} &= \mathcal{C}_{LS} + \mathcal{C}_{LS}^{\perp}
    = (\mathcal{C}_{NC1} + \mathcal{C}_{NC2,3}) + \mathcal{C}_{LS}^{\perp}.
\end{align}
Further, 
they provide empirical evidence that $\mathcal{C}_{LS}^{\perp}$ becomes negligible fast
during training.

\vspace{1em}
We will make some remarks comparing \eqref{NCcost} with the cost decomposition
\eqref{cost} and the results in this paper.
The last line of \eqref{cost} can be written in two different ways,
\begin{align}
   \mathcal{C} &= \sum_{j=1}^{Q}  \frac{1}{N_{j}} \sum_{i=1}^{N_{j}}
              \abs{W^{(L)} \Delta x_{0,j,i}^{(\tau, L-1)}}^{2} + \sum_{j=1}^{Q}
              \abs{W^{(L)} \overline{x_{0,j}^{(\tau, L-1)}} + b^{(L)} - y_{j}}^{2}
              \label{LS2} \\ 
               &= \sum_{j=1}^{Q}  \frac{1}{N_{j}} \sum_{i=1}^{N_{j}} \abs{W_{L} \Delta
               x_{0,j,i}^{(L-1)}}^{2} \,\, + \,\,\, \sum_{j=1}^{Q} \abs{W_{L} \overline{
           x_{0,j}^{(L-1)}} + b_{L} - y_{j}}^{2}.
           \label{LS1} 
\end{align}
If we let $b_{L}=0$ and $W_{L}=W_{LS}$ be the least squares matrix that minimizes the
second term in \eqref{LS1}, then $\mathcal{C} = \mathcal{C}_{LS}$, the
first term in \eqref{LS1} is $\mathcal{C}_{NC1}$, and the second term is
$\mathcal{C}_{NC2,3}$ (up to constants). Comparing with the interpretation given 
after \eqref{cost}, this suggests that the truncation maps in the hidden layers drive
NC1, while the last-layer affine map is responsible for NC2,3.

Let us compare the last-layer weights and biases obtained by minimizing according to
\eqref{LS1} and \eqref{LS2}. 
For the rest of this section we take $Y = \Id_{Q\times Q}$, $\overline{X^{(\tau,
\ell)}} := \left[ \overline{x_{0,1}^{(\tau, \ell)}} \cdots \overline{x_{0,Q}^{(\tau,
\ell)}} \right]$ and $\overline{X^{(
\ell)}} := \left[ \overline{x_{0,1}^{(\ell)}} \cdots \overline{x_{0,Q}^{(
\ell)}} \right]$.
In this work we have minimized with respect to \eqref{LS2},
\begin{align*}
    \tilde{W}^{(L)} = \pen{\overline{X^{(\tau, L-1)}}}, \, b^{(L)} = 0,
\end{align*}
which yields
\begin{align*}
    \tilde{W}_{L} &= \tilde{W}^{(L)} \pen{W^{(L-1)}} \nonumber \\
                  &= \pen{\overline{X^{(\tau, L-1)}}} \pen{W^{(L-1)}}.
\end{align*}
Minimizing with respect to \eqref{LS1}, as \citet{hanpapyandonoho22}, gives
\begin{align*}
    \label{WLS}
    W_{L} &= \pen{\overline{X^{(L-1)}}}, \, b_{L} = 0.
\end{align*}
We rewrite $W_{L}$: from Proposition \ref{prop-tau}, 
\begin{align*}
    \overline{X^{(L-1)}} &= W^{(L-1)}\overline{X^{(\tau, L-1)}} + B^{(L-1)} \nonumber \\
                         &= W^{(L-1)} \left( \overline{X^{(\tau, L-1)}} + \beta^{(L-1)}
                             \right)
\end{align*}
for
\begin{align*}
    \beta^{(L-1)} := \pen{W^{(L-1)}}B^{(L-1)},
\end{align*}
assuming $W^{(L-1)}$ is surjective.
Define the orthogonal projectors
\begin{alignat}{2}
    \mathcal{P}_{W} &:= \hspace{1.8em} \mathcal{P}_{(\ker W^{(L-1)})^{\perp}} 
                    &&= \pen{W^{(L-1)}}W^{(L-1)},\\ 
    \mathcal{P}_{X} &:= \mathcal{P}_{\ran(\overline{X^{(\tau,
    L-1)}} + \beta^{(L-1)})} 
                    &&= \left(\overline{X^{(\tau, L-1)}} + \beta^{(L-1)} \right)
                    \left(\overline{X^{(\tau, L-1)}} + \beta^{(L-1)}\right)^{+}.
\end{alignat}
Then
\begin{align*}
    W_{L} &= \left(W^{(L-1)} \left( \overline{X^{(\tau, L-1)}} + \beta^{(L-1)}
                             \right)\right)^{+} \nonumber \\
          &\stackrel{(1)}{=}  \left( \mathcal{P}_{W} \left[\overline{X^{(\tau, L-1)}} +
          \beta^{(L-1)}\right] \right)^{+} \left(W^{(L-1)} \mathcal{P}_{X} \right)^{+} ,
\end{align*}
where $(1)$ follows from a result that can be found in the book by \citet[Theorem
1.4.1]{cm-generalizedinverses}. Now, observe that
\begin{align}
    \label{ran-ker}
    \ran\left(\overline{X^{(\tau, L-1)}} + \beta^{(L-1)} \right) \subset \ran
    \pen{W^{(L-1)}} = \left( \ker W^{(L-1)}\right)^{\perp},
\end{align}
so
\begin{align*}
    \mathcal{P}_{W}\mathcal{P}_{X} = \mathcal{P}_{X}\mathcal{P}_{W} = \mathcal{P}_{X},
\end{align*}
and we conclude
\begin{align*}
    W_{L} &= \left(\overline{X^{(\tau, L-1)}} +
          \beta^{(L-1)}\right)^{+} \left(W^{(L-1)} \mathcal{P}_{X} \right)^{+} .
\end{align*}

If the sets in \eqref{ran-ker} are equal, then $\mathcal{P}_{W}=\mathcal{P}_{X}$ and
$W_{L}$ further simplifies to 
\begin{align*}
    W_{L} &= \left(\overline{X^{(\tau, L-1)}} +
          \beta^{(L-1)}\right)^{+} \left(W^{(L-1)} \right)^{+},
\end{align*}
so $W_{L}$ and $\tilde{W}_{L}$ would differ only in the $\beta^{(L-1)}$ term.
This could be achieved, for instance, by taking $d_{L-1} = Q$ and  $\overline{X^{(\tau,
L-1)}} \in \mathbb{R}^{M \times Q}$
injective (recall we are already assuming $W^{(L-1)}$ to be surjective). 
Suppose this is the case, for simplicity, and consider two subcases:
\begin{itemize}
    \item {First, if
    $b^{(L-1)} = 0$, then
    \begin{align*}
        W_{L} = \tilde{W}_{L} \text{ and } b^{(L)} = W_{L}b^{(L-1)} +
        b_{L} = b_{L} = 0
    \end{align*}
    and we are in the setting of
    \citet{hanpapyandonoho22}. This
    constrains the cone at the last nonlinear layer to be based at the origin.}
    \item{If we wish to consider strong neural
    collapse (SNC) as defined by \citet{mixonetal22}, then
    \footnote{Note that SNC is not reasonable if the penultimate layer features come from
    a ReLU neural network, as $\overline{x^{(L-1)}} = 0$ would imply that
    $x_{0,j,i}^{(L-1)}= 0$ for all $j=1, \cdots, Q$, $i=1, \cdots, N_{j}$. Still, the
study of the ranks of the matrices and comparison with the first subcase is instructive.}
    \begin{align*}
        \overline{x^{(L-1)}} = Q^{-1} \sum_{j=1}^{Q} \overline{x_{0,j}^{(L-1)}} =0,
    \end{align*}
    which implies that $\overline{X^{(L-1)}} \in \mathbb{R}^{Q\times Q}$ cannot be full
    rank. To conserve injectivity of
    $\overline{X^{(\tau, L-1)}}$, we must have $b^{(L-1)} \neq 0$, and therefore
    $\tilde{W}_{L} \neq W_{L}$.
    Indeed, SNC assumes
    \begin{align*}
        W_{L} W_{L}^{T} = \sqrt{Q} \left(\Id_{Q\times Q} - \frac{1}{Q}
        u_{Q}u_{Q}^{T}\right),\quad
        b_{L} = \frac{1}{Q} u_{Q},
    \end{align*}    
    and in that case not only is the last-layer weight matrix not the least squares matrix
    \eqref{WLS}, it is not surjective. 
    To fit this, the constructions in Proposition
    \ref{decreasingwidths} and Theorem \ref{mainthm} could be modified so that $W^{(L)}\in
    \mathbb{R}^{Q \times M}$ has rank $Q-1$
    while still achieving
    zero loss:
    \begin{align*} 
        W^{(L)} = \left(\Id_{Q\times Q} - B^{(L)}\right) \overline{X^{(\tau,
        L-1)}}^{+}, \quad b^{(L)} = \frac{1}{Q}u_{Q}.
    \end{align*}
    }
\end{itemize}

The remarks in this
section indicate that the constructions presented in this paper align well with the
framework of neural collapse, at least in certain cases or with minor adaptations.
Therefore, employing the truncation map as a tool to investigate the analytic and
geometric properties of deep neural networks could also be advantageous in this context.

\section{Proof of Proposition \ref{decreasingwidths}} \label{sec:proof}

To prove the proposition, we use barycentric coordinates. We consider a point
$x\in \mathbb{R}^{M}$ to be represented by
\begin{align}
    \label{bary}
    x = \sum_{j=1}^{Q} \kappa_{j} \overline{x_{0,j}} + \tilde{x}
\end{align}
where $\tilde{x}\in (\spn\{\overline{x_{0,j}}\}_{j=1}^{Q})^{\perp}$, and
the means $\overline{x_{0,j}}$ form the vertices of a $Q$-simplex for which $\overline{x}$
is the barycenter (see Figure \ref{fig:simplex}). Switching to these coordinates can be done when preprocessing the data.

\vspace{1em}

\begin{proof}
   Assume the data $\mathcal{X}_{0}$ is given in barycentric coordinates. 
   We wish to construct $W^{(\ell)}, b^{(\ell)}$, $\ell =1, \cdots, Q$,  such that
   $\Delta X_{0}^{(\tau , Q)}=0$, and the last layer can be explicitly solved to
   achieve $X^{(L)}=Y^{ext}$.

   Let
   \begin{align*}
       \theta_{*} := \theta_{0} + \max_{j} \theta_{*,j},
   \end{align*}
   for some $\theta_{0} \geq 0$ such that $\theta_{*} < \pi$. For convenience we also
   require that $\theta_{0}$ is such that $\theta_{*} \geq \frac{\pi}{2}$.
   Recall from \eqref{P(l)}
   \begin{align*}
       P^{(\ell )} &= \pen{W^{(\ell )}} W^{(\ell )} \pen{W^{(\ell -1)}} W^{(\ell -1)}
       \cdots \pen{W^{(1)}}W^{(1)} \\ \nonumber
                   &= \pen{W^{(\ell )}} W^{(\ell )} P^{(\ell -1)}.
   \end{align*}
   
   \vspace{1ex}
   \textbf{Induction base \texorpdfstring{$\ell =1$}{l=1}.}
   If we choose 
   \begin{align*}
       W_1 = W^{(1)} := \P{d_0}{d_1} W_{\theta_{*}} R_1,
   \end{align*}
   for $W_{\theta_{*}}$ defined as in \eqref{Wtheta}, and $R_1\in \SO(d_0,\mathbb{R})$ such
   that 
   \begin{align*}
       R_1 f_1 = \frac{u_{d_0}}{\abs{u_{d_0}}},
   \end{align*}
   and also
   \begin{align*}
       b_1 = b^{(1)} := -W_1 (\overline{x_{0,1}} + \mu_1 f_1),
   \end{align*}
   for some $\mu_1$,
   then by Lemma \ref{lemma}, 
   \begin{align*}
      \tau _{W^{(1)}, b^{(1)}}(x) =
      \begin{cases}
          P^{(1)} x, \quad &x \in (\overline{x_{0,1}} + \mu_1 f_1) + \mathfrak{C}_{\theta_{1}}[f_1], \\
          P^{(1)} (\overline{x_{0,1}} + \mu_1 f_1) , \quad &x \in (\overline{x_{0,1}} + \mu_1 f_1) +
          \mathfrak{C}_{\theta_{1}}[-f_1].
      \end{cases}
   \end{align*}
   Thus, we wish to find $\mu_1$ such that $\mathcal{X}_{0,1}$ is entirely contained in
   the backward cone, and $\bigcup_{j\neq 1} \mathcal{X}_{0,j}$ is entirely contained in
   the forward cone.
   We have chosen $\theta_{*}$ such that
   \begin{align*}
       \bigcup_{j\neq 1} B_{4 \delta}( \overline{x_{0,j}}) \subset \overline{x_{0,1}} +
       \mathfrak{C}_{\theta_{*}}[f_{1}],
   \end{align*}
   and so $\mu_1 < 3 \delta$ suffices\footnote{This is not optimal.} for
   \begin{align*}
       \bigcup_{j\neq 1} B_{\delta}( \overline{x_{0,j}}) \subset (\overline{x_{0,1}} +
       \mu_1 f_1) + \mathfrak{C}_{\theta_{*}}[f_{1}].
   \end{align*}
   Then, as we have also assumed that $\theta_{*} \geq \frac{\pi}{2}$, and since in
   general\footnote{See appendix.}
   \begin{align*}
       B_{\delta}(0) \subset - 2 \delta f + \mathfrak{C}_{\frac{\pi}{2}}[f]
   \end{align*} 
   for any $\delta >0$ and unit vector $f \in \mathbb{R}^{d_0}$, 
   it is sufficient to assume that $\mu_1 > 2 \delta$ for
   \begin{align*}
       B_{\delta}( \overline{x_{0,1}}) \subset (\overline{x_{0,1}} + \mu_1 f_1) +
       \mathfrak{C}_{\theta_{*}}[-f_{1}].
   \end{align*}
   Hence, for $\mu_1 \in (2 \delta,\, 3 \delta)$, 
   \begin{align*}
      \tau _{W^{(1)}, b^{(1)}}(X_{0,j}) =
      \begin{cases}
          P^{(1)} X_{0,j}, \quad &j >1,\\
          P^{(1)} (\overline{x_{0,1}} + \mu_1 f_1) u_{N_1}^{T}, \quad &j=1.
      \end{cases}
   \end{align*}
   
   \vspace{1ex}
   \textbf{The induction step \texorpdfstring{$\ell -1 \to \ell$}{l-1 -> l}.}
   Assume 
   \begin{align}
       \label{ind-hyp1}
       X_{0,j}^{(\tau ,\ell -1)} = 
       \begin{cases}
           P^{(\ell-1)} (\overline{x_{0,j}} + \mu_{j} f_{j})  u_{N_{j}}^{T}, \quad &j \leq \ell -1, \\
           P^{(\ell-1)} X_{0,j}, \quad &j > \ell -1,
       \end{cases}
   \end{align}
   for some collection $\{\mu_{j}\}_{j=1}^{\ell -1} \subset (2 \delta, \, 3 \delta )$.  We
   will construct $W^{(\ell)}, b^{(\ell)}$ such that
   \begin{align}
      \label{ind-concl1}
      X_{0,j}^{(\tau ,\ell)} = \tau_{W^{(\ell )}, b^{(\ell )}} (X_{0,j}^{(\tau, \ell-1)}) = 
      \begin{cases}
          P^{(\ell)} (\overline{x_{0,j}} + \mu_{j} f_{j}) u_{N_{j}}^{T} 
          , \quad &j < \ell , \\ 
          P^{(\ell)} (\overline{x_{0,\ell }} + \mu_{\ell } f_{\ell }) u_{N_{\ell}}^{T},
          \quad &j=\ell, \\
          P^{(\ell)} X_{0,j} 
          , \quad &j > \ell.
      \end{cases}
   \end{align}
   
   Let 
   \begin{align}
       \label{Wellproofc}
       W^{(\ell )} &:= \P{d_0}{d_{\ell}} W_{\theta_{*}} R_{\ell}, \nonumber \\
       b^{(\ell )} &:= - W^{(\ell )}\, (\overline{x_{0,\ell }} + \mu_{\ell } f_{\ell })  ,
   \end{align}
   for $\P{d_0}{d_{\ell}}$ the projection defined in \eqref{proj}, $\mu_{\ell} \in (2
       \delta, \, 3 \delta)$ and $R_{\ell} \in \SO(d_0, \mathbb{R})$ such that
   \begin{align*}
       R_{\ell } h_{\ell } = \frac{u_{d_0}}{\abs{u_{d_0}}}.
   \end{align*}
   With these definitions and Lemma \ref{lemma}, we have the following truncation map with
   its associated cones,
   \begin{align*}
      \tau _{W^{(\ell)}, b^{(\ell)}}(x) =
      \begin{cases}
          \pen{W^{(\ell)}} W^{(\ell)}\, x, \quad &x \in (\overline{x_{0,\ell }} +
          \mu_{\ell } f_{\ell }) + \mathfrak{C}_{\theta_{\ell}}[f_{\ell}], \\
          \pen{W^{(\ell)}} W^{(\ell)}\, (\overline{x_{0,\ell }} + \mu_{\ell } f_{\ell }),
          \quad &x \in (\overline{x_{0,\ell }} + \mu_{\ell } f_{\ell}) +
      \mathfrak{C}_{\theta_{\ell}}[-f_{\ell}].  
      \end{cases}
   \end{align*}
   Just as in the base case, $\mu_{\ell}< 3 \delta$ implies
   \begin{align}
       \label{fcone}
       \bigcup_{j\neq \ell} B_{\delta}( \overline{x_{0,j}}) \subset (\overline{x_{0,\ell}}
       + \mu_{\ell} f_{\ell}) + \mathfrak{C}_{\theta_{*}}[f_{\ell}],
   \end{align}
   so that $\bigcup_{j\neq \ell} \mathcal{X}_{0,j}$ sits in the forward cone, and
   $\mu_{\ell } > 2 \delta$ implies
   \begin{align}
       \label{bcone}
       B_{\delta}(\overline{x_{0,\ell}}) \subset (\overline{x_{0,\ell}} + \mu_{\ell}
       f_{\ell}) + \mathfrak{C}_{\theta_{*}}[-f_{\ell}].
   \end{align}
   so that $\mathcal{X}_{0,\ell}$ sits inside the backward cone. However, now that the
   data has already been truncated at least once, it takes a little more work to justify
   how $\tau_{W^{(\ell)}, b^{(\ell)}}$ acts on each $X_{0,j}^{(\tau,\ell -1)}$.

   For $j < \ell$, we have from \eqref{ind-hyp1} that
   \begin{align*}
       \mathcal{X}_{0,j}^{(\tau ,\ell -1)} = P^{(\ell -1)}( \overline{x_{0,j}} +
       \mu_{j}f_{j}),
   \end{align*}
   and since $W^{(\ell)}P^{(\ell -1)}=W^{(\ell)}$ by \eqref{WP=W} in Lemma
   \ref{penrose-concat}, we can use \eqref{tauP} to get
   \begin{align*}
       \tau_{W^{(\ell)}, b^{(\ell)}}(P^{(\ell -1)}( \overline{x_{0,j}} +
       \mu_{j}f_{j})) = \tau_{W^{(\ell)}, b^{(\ell)}}(\overline{x_{0,j}} +
       \mu_{j}f_{j}).
   \end{align*}
   We claim that 
   \begin{align}
       \label{j<lclaim}
       ( \overline{x_{0,j}}+ \mu_{j}f_{j}) \in ( \overline{x_{0,\ell }} + \mu_{\ell
       }f_{\ell }) + \mathfrak{C}_{\theta_{*}}[f_{\ell }].
   \end{align}
   Indeed, 
   note that for all $j'=1, \cdots, Q$,
   \begin{align}
       \label{xbar2}
       \overline{x} = \overline{x_{0,j'}} + \abs{\overline{x_{0,j'}} -\overline{x}}f_{j'},
   \end{align}
   so it is clear that
   \begin{align*}
       \overline{x} \in \overline{x_{0,\ell }} + \mathfrak{C}_{\theta_{*}}[f_{\ell }],
   \end{align*}
   and if $\mu < \abs{\overline{x_{0,\ell }} - \overline{x}}$, then also
   \begin{align}
       \label{xbarcone}
       \overline{x} \in ( \overline{x_{0,\ell }} + \mu f_{\ell }) + \mathfrak{C}_{\theta
       _{*}}[f_{\ell }].
   \end{align}
   By assumption \eqref{clustercond1}, 
   \begin{align*}
       \delta < \frac{1}{4} \min_{j}\abs{\overline{x_{0,j}} - \overline{x}},
   \end{align*}
   so in particular
   \begin{align*}
       \mu_{\ell } < 3 \delta < \abs{\overline{x_{0,\ell }} - \overline{x}}.
   \end{align*}
   Again by \eqref{clustercond1} and \eqref{xbar2},
   \begin{align*}
       \mu_{j} < 3 \delta < \abs{\overline{x_{0,j}} - \overline{x}}
   \end{align*}
   and $(\overline{x_{0,j}} + \mu_{j}f_{j})$ is in the line segment between
   $\overline{x_{0,j}}$ and $\overline{x}$, which is entirely contained in the forward
   cone by \eqref{fcone}, \eqref{xbarcone} and the convexity of the cone. Having proved
   \eqref{j<lclaim}, 
   \begin{align*}
       X_{0,j}^{(\tau ,\ell )} &= \tau_{W^{(\ell)},b^{(\ell)}} \left(X_{0,j}^{(\tau, \ell
       -1)}\right)
       \nonumber \\ 
                               &= \pen{W^{(\ell)}} W^{(\ell)} X_{0,j}^{(\tau , \ell-1)}
       \nonumber \\
                               &= \pen{W^{(\ell)}} W^{(\ell)}P^{(\ell -1)}
                               (\overline{x_{0,j}}+ \mu_{j} f_{j}) u_{N_{j}}^{T} 
                               \\
                               &= P^{(\ell)} (\overline{x_{0,j}}+ \mu_{j} f_{j})
                               u_{N_{j}}^{T}. \nonumber 
   \end{align*}

   For $j = \ell$, we have assumed that
   \begin{align*}
       \mathcal{X}_{0,\ell }^{(\tau ,\ell -1)} = P^{(\ell -1)}\mathcal{X}_{0,\ell },
   \end{align*}
   and once again by \eqref{WP=W}, 
   \begin{align*}
       \tau_{W^{(\ell )},b^{(\ell)}}\left( X_{0,\ell }^{(\tau, \ell -1)}\right) =
       \tau_{W^{(\ell)}, b^{(\ell)}}(X_{0,\ell })
   \end{align*}
   and we can use \eqref{bcone} in
   \begin{align*}
       X_{0,\ell}^{(\tau ,\ell )} &= \tau_{W^{(\ell)},b^{(\ell)}} \left(X_{0,\ell}^{(\tau,
       \ell -1)}\right) \nonumber \\
                                  &= \pen{W^{(\ell)}} W^{(\ell)}
                               (\overline{x_{0,\ell }}+ \mu_{\ell } f_{\ell }) u_{N_{\ell
                               }}^{T}  \\
                               &\stackrel{(*)}{=}  P^{(\ell)}
                               (\overline{x_{0,\ell }}+ \mu_{\ell } f_{\ell }) u_{N_{\ell
                               }}^{T}, \nonumber 
   \end{align*}
   where $(*)$ follows from Lemma \ref{penrose-concat}.

   Similarly, for $j > \ell$,
   \begin{align*}
       \mathcal{X}_{0,j}^{(\tau,\ell -1)} = P^{(\ell -1)}\mathcal{X}_{0,j},
   \end{align*}
   and we can use \eqref{WP=W} and \eqref{fcone} to compute
   \begin{align*}
       X_{0,j}^{(\tau ,\ell )}  &= \tau_{W^{(\ell)},b^{(\ell)}} \left(X_{0,j}^{(\tau, \ell
   -1)}\right)
       \nonumber \\  
                                &= \pen{W^{(\ell)}} W^{(\ell)} X_{0,j}^{(\tau , \ell-1)}
       \nonumber \\
                               &= \pen{W^{(\ell)}} W^{(\ell)} P^{(\ell -1)} X_{0,j} \\
                               &= P^{(\ell)} X_{0,j},\nonumber 
   \end{align*}
   and we have proved \eqref{ind-concl1}.

   \vspace{1ex}
   \textbf{Conclusion of proof.}
   After $L-1=Q$ steps, 
   \begin{align*}
       X_{0,j}^{(\tau ,Q)} =\overline{X_{0,j}^{(\tau ,Q)}} =  P^{(Q)}(
       \overline{x_{0,j}} - \mu_{j} f_{j}) u_{N_{j}}^{T},
   \end{align*}
   so $\Delta X_{0,j}^{(\tau, Q)}= 0.$ In the last layer, we have from 
   \eqref{taulast}
   \begin{align*}
       X^{(L)} = W^{(Q+1)}\overline{X_0^{(\tau, Q)}} + B^{(Q+1)}.
   \end{align*}
   It is easy to check that if $(\overline{X_0^{red}})^{(\tau,Q)}$ is full rank, then
   \begin{align*}
       W^{(Q+1)} &= Y \pen{(\overline{X_0^{red}})^{(\tau,Q)}},\nonumber \\ 
       b^{(Q+1)} &= 0,  
   \end{align*}
   solve
   \begin{align}
       \label{WQ+1proofc}
        W^{(Q+1)}\overline{X_0^{(\tau, Q)}} + B^{(Q+1)} = Y^{ext},
   \end{align}
   so that the cumulative parameters \eqref{Wellproofc} and \eqref{WQ+1proofc} define a
   neural network with weights and biases  $(W_{i}^{**}, b_{i}^{**})_{i=1}^{Q+1}$ such
   that
   \begin{align*}
       X^{(L)} = Y^{ext}
   \end{align*}
   and 
   \begin{align*}
       \mathcal{C}[(W_{i}^{**}, b_{i}^{**})_{i=1}^{Q+1}] = 0.
   \end{align*}

   \vspace{1ex}
   It remains to check that
   $\overline{X_0^{red}}^{(\tau,Q)} \in \mathbb{R}^{M\times Q}$ is indeed full rank, which
   is the same as checking that the set $\left\{ P^{(Q)}( \overline{x_{0,j}} - \mu_{j}
   f_{j})\right\}_{j=1}^{Q}$ is linearly independent.
   First, note that the set of vectors $\left\{(\overline{x_{0,j}} - \mu_{j}
   f_{j})\right\}_{j=1}^{Q}$ is linearly independent: We have assumed that
   $\overline{X_0^{red}}$ is full rank, so the means $\overline{x_{0,j}}$ form the
   vertices of a
   $Q$-simplex which spans a $Q$-dimensional subspace of $\mathbb{R}^{M}$. Each
   $-\mu_{j}f_{j}$ simply shifts a vertex of this simplex towards the barycenter
   $\overline{x}$ in such a way that  $\left\{(\overline{x_{0,j}} - \mu_{j}
   f_{j})\right\}_{j=1}^{Q}$ forms a smaller simplex spanning the same $Q$-dimensional
   subspace.

   Next, we need $P^{(Q)}$ to be suitably rank preserving.
   This can be done by changing the matrices $\tilde{R}$ and $R_{Q}$ in the definitions of
   $W_{\theta_{*}}$ and $W^{(Q)}$, respectively, as we explain now.
   From \eqref{P(l)=PenWW} in Lemma \ref{penrose-concat},
   \begin{align*}
       P^{(Q)} &= \pen{W^{(Q)}} W^{(Q)} \nonumber \\
               &= (W^{(Q)})^{T} \left( W^{(Q)} (W^{(Q)})^{T} \right)^{-1} W^{(Q)} \\
               &= (W_{\theta_{*}} R_{Q})^{T} \P{d_0}{d_{Q}}^{T} (W^{(Q)}W^{(Q)\,T})^{-1}
               \P{d_0}{d_{Q}} W_{\theta_{*}} R_{Q}, \nonumber
   \end{align*}
   then $W_{\theta_{*}}R_{Q}\in GL(d_0, \mathbb{R})$, and $\P{d_0}{d_{Q}}^{T}$ is an
   inclusion, and so
   \begin{align*}
       (W_{\theta_{*}}R_{Q})^{T} \P{d_0}{d_{Q}}^{T} (W^{(Q)}W^{(Q)\,T})^{-1}
   \end{align*}
   does not change the rank of any matrix on which it acts. 

   We now turn to the action of $\P{d_0}{d_{Q}} W_{\theta_{*}}R_{Q}$.
   As we are working in barycentric coordinates \eqref{bary}, we only care about the
   action of $P^{(Q)}$ on the first $Q$ coordinates of any vector, 
   and so restrict our attention to the first $Q$
   columns of $\P{d_0}{d_{Q}} W_{\theta_{*}}R_{Q}$. 
   Moreover, $\P{d_0}{d_{Q}}$ acts on
   $W_{\theta_{*}}R_{Q}$ by cutting off all but the first $d_{Q}\geq Q$ rows. Altogether, 
   it is sufficient that the first $Q \times Q$ block of $W_{\theta_{*}} R_{Q}$ be
   invertible. Using the Laplace cofactor expansion of the determinant, it can be shown
   that there exists a permutation of the rows of $W_{\theta_{*}}R_{Q}$ that achieves
   this.
   
   Let $P$ be the elementary matrix that realizes this permutation, and note
   that $P\in O(d_0, \mathbb{R})$. Then
   letting $\lambda:= \lambda(\theta_{*}, \theta_{d_0})$ in the
   following for notational convenience, and expanding the definition of $W_{\theta_{*}}$
   from \eqref{Wtheta},
   \begin{align*}
       P W_{\theta_{*}}R_{Q} &= P \tilde{R} \diag(1, \lambda, \cdots, \lambda)
       \tilde{R}^{T} R_{Q} \nonumber \\ 
                             &= P\tilde{R} \diag(1, \lambda, \cdots, \lambda)
       \tilde{R}^{T} (P^{T}P) R_{Q} \\
                   &= (P\tilde{R}) \diag(1, \lambda, \cdots, \lambda) (P \tilde{R})^{T} (P
                   R_{Q}) \nonumber
   \end{align*}
   and so the permutation needed amounts to permuting the rows of $\tilde{R}$ and $R_{Q}$.
   Recall the definitions
   \begin{align}
       \label{Rconstr}
       \tilde{R}e^{d_0}_1 = \frac{u_{d_0}}{\abs{u_{d_0}}} = R_{Q} f_{Q}.
   \end{align}
   Clearly, $P u_{d_0} = u_{d_0}$, and so we can pick $\tilde{R}$ and $R_{Q}$ satisfying
   \eqref{Rconstr} and such that the block $(W_{\theta_{*}}R_{Q})_{Q\times Q}$ is
   invertible, and $P^{(Q)}$ restricted to $\spn \left\{
   \overline{x_{0,j}}\right\}_{j=1}^{Q}$ is injective.

   \vspace{1ex}
   Finally, note that the global minimum we have constructed is degenerate, since we have a
   family of cumulative bias terms $(b_{*}^{(\ell)}[\mu_{\ell}])_{\ell=1}^{Q}$ 
   parametrized by $(\mu_{\ell})_{\ell =1}^{Q} \subset (2 \delta, 3 \delta )^{Q}$.
\end{proof}

\vspace{1em}
\noindent
{\bf Acknowledgments:}
T.C. gratefully acknowledges support by the NSF through the grant DMS-2009800, and the RTG
Grant DMS-1840314 - {\em Analysis of PDE}. P.M.E. was supported by NSF grant DMS-2009800
through T.C., and UT Austin's GS Summer fellowship.

\appendix
\section{Geometry} \label{appendix}

As usual, we define the angle between two vectors $u, v, \in \mathbb{R}^{n}$ to be 
\begin{align*}
    \angle (u,v) := \arccos \frac{\langle u,v\rangle }{\norm{u}\norm{v}}.
\end{align*}
Then 
\begin{align*}
    \angle(u,v) \leq \theta
\end{align*}
is equivalent to 
\begin{align*}
    \frac{\langle u,v\rangle }{\norm{u}\norm{v}} \geq \cos \theta.
\end{align*}

A cross section of the largest cone centered around the diagonal in $\mathbb{R}^{n}$ is a
closed $(n-1)$-dimensional disk inscribed into the $n$-simplex. The angle $\theta_{n}$ of
such a cone can be obtained by analyzing a point in its boundary, such as the barycenter
of any one of
the faces. That is, if we choose the $j$-th face, we consider
\begin{align*}
    v_{j} = \sum_{i\neq j} \frac{1}{n-1} e_{i}
\end{align*}
In this case
\begin{align*}
    \norm{u_n} = \sqrt{n}, \quad \norm{v_{j}} = \frac{1}{\sqrt{n-1}}, \nonumber \\
    \langle u_n, v_{j}\rangle = (n-1)\frac{1}{n-1} = 1,
\end{align*}
so that
\begin{align*}
    \angle(u_n,v_{j}) = \arccos \frac{\sqrt{n-1}}{\sqrt{n}},
\end{align*}
giving the expression \eqref{thetan} for $\theta_{n}$.

In the proof of Proposition \ref{decreasingwidths}, we make the assumption $\theta_{*}
\geq \frac{\pi}{2}$. Note that a cone based at the origin with axis some unit vector
$f \in \mathbb{R}^{n}$
will contain the ball $B_{\delta}(2 \delta f)$ if it contains 
\begin{align*}
    \delta f + \delta f^{\perp}
\end{align*}
for any $f^{\perp}$ a unit vector perpendicular to $f$. The smallest angle for which this
holds is given by
\begin{align*}
    \cos(\theta) &= \frac{\langle \delta f, \delta f + \delta f^{\perp}\rangle }{\norm{\delta f}
        \norm{\delta f + \delta f^{\perp}}} \nonumber \\  
                 &= \frac{\norm{\delta f}}{\sqrt{\norm{\delta f}^{2} + \norm{\delta
                 f^{\perp}}^{2}}} \\ 
                 &= \frac{1}{\sqrt{2}},\nonumber 
\end{align*}
or $\theta = \frac{\pi}{2}$.


\end{document}